\documentclass{article}
\usepackage{xcolor}
\usepackage{times}

\usepackage{amsfonts}
\usepackage{times}
\usepackage[pdftex]{graphicx}
\usepackage{latexsym}
\usepackage{amssymb}
\usepackage{amsmath}
\usepackage{verbatim}
\newtheorem{theorem}{Theorem}
\newtheorem{assumption}[theorem]{Assumption}
\newtheorem{corollary}[theorem]{Corollary}
\newtheorem{lemma}[theorem]{Lemma}
\newtheorem{prop}[theorem]{Proposition}
\newtheorem{remark}[theorem]{Remark}
\newenvironment{proof}{ \textbf{Proof:} }{ \hfill $\Box$}
\def\opt{\mathsf{OPT}}
\def\bbE{{\mathbb{E}}}
\def\bbP{{\mathbb{P}}}

\def\cA{\mathcal{A}}

\def\cP{\mathcal{P}}
\def\cQ{\mathcal{Q}}
\def\cR{\mathcal{R}}
\def\sfc{{\mathsf{c}}}
\def\sfp{{\mathsf{p}}}
\def\nn{\nonumber}

\usepackage{algorithm}
\usepackage{algorithmic}
\usepackage{caption}
\usepackage{natbib}
\usepackage{xcolor}
\usepackage{subfig}
\usepackage{caption}

\newcommand{\EE}[1]{\mathbb{E}\left[#1\right]}

\newcommand{\Regret}{\mathcal{R}}

\newcommand{\dt}{\Delta t}

	\title{Continuous Time Bandits With Sampling Costs}
\author{Rahul Vaze \and Manjesh Hanawal }
\begin{document}

\maketitle

\begin{abstract}
We consider a continuous time multi-arm bandit problem (CTMAB),  where the learner can sample arms any number of times in a given interval and obtain a random reward from each sample, however, increasing the frequency of sampling incurs 
an additive penalty/cost. Thus, there is a tradeoff between obtaining large reward and incurring sampling cost as a function of the sampling frequency. The goal is to design a learning algorithm that 
minimizes the regret, that is defined as the difference of the payoff of the oracle policy and that of the learning algorithm.  
We establish lower bounds on the regret achievable with any algorithm, and propose algorithms that achieve the lower bound up to logarithmic factors. For the single arm case, we show that the lower bound on the regret is $\Omega(1/\mu)$,  and an upper bound with regret $O((\log (T/\lambda))^2/\mu)$, where $\mu$ is the mean of the arm,  $T$ is the time horizon, and $\lambda$ is the tradeoff parameter between the payoff and the sampling cost. With $K$ arms, we show that the lower bound on the regret is $\Omega( K\mu[1]/\Delta^2)$, and an upper bound $O( K (\log (T/\lambda))^2 \mu[1]/\Delta^2)$ where $\mu[1]$ now represents the mean of the best arm, and $\Delta$ is the difference of the mean of the best and the second-best arm. 
%The optimal regret is $\Theta(1/\mu)$ for the single arm case, when the mean of the single arm is $\mu
\end{abstract}

%\begin{keywords}%
%Multi-arm bandits, best arm identification, optimal regret,  mean estimation sampling cost.%\keywords{Bandit algorithms, sampling cost, best arm identification, mean estimation, 
%
%\end{keywords}

%\section{Introduction}
\label{sec:Intro}

\section{Introduction}
\label{sec:Intro}The classical discrete-time multi-arm bandit (DMAB) is a versatile learning problem  \cite{NOW12_bubeck2012regret,Book2019_BanditAlgorithms} that has been extensively studied in literature. 
By discrete-time, we mean that there are a total of $T$ discrete slots, and in each slot, a learning algorithm can choose to `play' any one of the possible $K$ 
arms. 
%AISTATSIf the learning algorithm plays arm $k$ in slot $t$, it 
%gets a random reward with mean $\mu[k]$ independent of the slot index. The total reward of the learning algorithm at the end of the $T$ slots 
%is the accumulated random reward over the $T$ slots, which is compared against the reward of the oracle policy which knows the true means of the $K$ arms 
%throughout, and thus always plays the best arm (arm with the highest mean). 

%AISTATSThe typically considered performance metric is  {\it regret}, that 
%is defined as the expected difference of the total reward of the oracle policy and that of a learning algorithm. 
%The optimal regret for the DMAB problem is known to be $\Theta\left(\log(T)/\Delta^2\right)$, when the rewards are Bernoulli distributed or follow any sub-Gaussian distribution, and $\Delta$ is the sub-optimality gap. Multiple algorithms such 
%as UCB \cite{auer2002finite} and Thompson sampling \cite{agrawal2012analysis} are known to achieve the asymptotic optimal regret. 

In this paper, we consider a continuous-time multi-arm bandit problem (CTMAB) that is well motivated from pertinent applications 
discussed later in this section. In particular, in CTMAB,  the total time horizon is $T$, and there are $K$ arms. The main distinction between the 
DMAB and the CTMAB  is that with the CTMAB,  an arm can be sampled/played at any (continuous) time $t$ before $T$. 
Once the sampling time $t \in [0,T]$ is selected, similar to the DMAB problem, any one of $K$ arms can be played, and if arm $k$ is played at time $t$, the learning algorithm gets a random reward with 
mean $\mu[k]$ independent of the time $t$. Without loss of generality, we let $1> \mu[1] > \mu[2] \dots > \mu[K]$.

Without any other restriction, per se, any algorithm can play infinite number of arms in time horizon $T$ by repeatedly playing arms at arbitrarily small intervals. Thus, to make the problem meaningful,  we account for the sampling (arm playing) cost that depends on how often any arm is sampled. Specifically, if two consecutive plays (of any two arms) are made at time $t$ and $t+\dt$, then the {\bf sampling cost} 
for interval $\dt$ is $f(\dt)$, where $f$ is a decreasing function. 
The total sampling cost is defined as the sum of sampling cost over the total time horizon $T$. 
Sampling cost penalizes high frequency sampling i.e., the higher the frequency of consecutive plays, the higher is the sampling cost. 
Considering sampling cost depending on consecutive plays of a specific arm results in multiple decoupled single arm CTMAB problems, thus a special case of the considered problem.

The overall payoff in CTMAB  is then defined as the accumulated random reward obtained from each sample minus $\lambda$ times the total sampling cost over the time horizon $T$. The variable $\lambda$ represents the relative weight of the two costs.
The regret is defined as usual, the expected difference between the overall payoff of the oracle policy and that of a learning 
algorithm.
There is a natural tradeoff in CTMAB, higher frequency sampling increases the accumulated reward but also increases the sampling cost at the same time. 
Compared to DMAB, where an algorithm has to decide which arm to play next at each discrete time,
in CTMAB, there are two decisions to be made, given the history of decisions and the current payoff (reward minus the sampling cost); i) which arm to play next, and ii) at what time.  

%AISTATS It is worth noting that CTMAB cannot be posed as a special case of DMAB, since in DMAB with $T$ time slots, the number of samples that can be obtained is at most $T$. Whereas, in CTMAB, there is no exogenous restriction on the number of samples an algorithm can obtain in a time horizon $T$, and we get a {\bf general model to capture the tradeoff between the reward and the sampling cost}. 
%In fact, DMAB can be posed as a special case of CTMAB by putting appropriate restrictions on the rate of sampling and ignoring the sampling cost. 

We next discuss some practical motivations for considering the CTMAB.
%give two motivating applications from crowd-sourcing and queuing systems to study the CTMAB. \\
We first motivate the CTMAB in the single arm case ($K=1$) which in itself is non-trivial. 
Consider that there is a single agent (human/machine)  that processes same types of jobs (e.g. working in call center, data entry operation, sorting/classification job etc.) with an inherent random quality of work $q$, and random per-job (unknown) utility $u(q)$. The agent demands a payment depending on how frequently it is asked to accomplish a job \cite{gong2019truthful}. Alternatively, as argued in \cite{gopalakrishnan2016routing}, the quality of agents' work suffers depending on the speed of operation/load experienced. Thus, the payoff  is the total accumulated utility  minus the payment or the speed penalty (that depends on the frequency of work), and the objective is to find the optimal speed of operation, to maximize the overall payoff. 

To motivate the CTMAB with multiple arms, the natural extension of this model is to consider a platform or aggregator that has multiple agents, each having a  random work quality and corresponding per-job utility. 
The random utility of any agent can be estimated by assigning  jobs to it and observing its outputs. The platform charges a cost depending on the speed (frequency) of job arrivals to the platform, that is indifferent to the actual choice of the agent being requested for processing a particular job. Given a large number of jobs to be processed in a finite time, the objective is to maximize the overall payoff;  the total accumulated utility minus the payment made to the platform.
When the platform cost is the sum of the cost across agents, where each agent cost depends on the rate at which it processes  jobs, the problem decouples into multiple single arm CTMAB problems.

In this paper, for the ease of exposition, we assume that the sampling cost function $f$ is $f(x) = \frac{1}{x}$, i.e., if two consecutive plays are made at time $t$ and $t+\dt$, then the sampling cost 
for interval $\dt$ is $\frac{1}{\dt}$, which is intuitively appealing and satisfies natural constraints on $f$. How to extend results for general functions is discussed in Remark \ref{rem:fconvex}. 
Under this sampling cost function, assuming arm $1$ has the highest mean $\mu[1]$, it turns out (Proposition \ref{prop:oracle}) that the oracle policy always plays the best arm (arm with the highest mean) $N^\star_T = \Theta\left(\mu[1]T/\lambda\right)$ times at equal intervals in interval $[0,T]$. Importantly, the number of samples (or sampling frequency) obtained by the oracle policy {\bf depends on the mean of the best arm}.

This dependence of the oracle policy's choice of the sampling frequency on the mean of the best arm results in two main distinguishing features of the CTMAB  compared to the DMAB problem, described as follows. CTMAB  is non-trivial even when there is only a single arm unlike the DMAB problem, where it is trivial. The 
non-triviality arises since the learning algorithm for the CTMAB  has to choose the time at which to obtain the next sample that depends on the mean of that arm, which is also unknown.
Moreover with CTMAB, it is not enough to identify the optimal arm, the quality of the estimates is equally important as the sampling cost depends on that. In contrast, with DMAB, it is sufficient for an algorithm to identify the right ordering of the arms.

Recall that we have assumed $\mu[1]<1$. The case $\mu[k]\ge 1$ will follow similarly, where it is worth noting that the setting of $\mu[k]\ge 1$ is easier than when $\mu[1]<1$,  since estimates of $\mu[1]$ have to be accurately estimated in CTMAB, and that becomes harder as $\mu[1]$ decreases.  

{\bf Our contributions for the CTMAB  are as follows.} 

1. For the single arm CTMAB,  where $\mu = \mu[1]$, we propose an algorithm whose regret is at most $O\left(\frac{1}{\mu}(\log(T/\lambda))^2\right)$. In converse, we show that for any online algorithm that uses only unbiased estimators of $\mu$ for making its decisions, its regret is $\Omega\left(\frac{1}{\mu}\right)$.  
%Exact result is little more involved and presented in Theorem \ref{thm:lbsinglearm}. 
The reason for considering the unbiased restriction is that the sample mean has the minimum variance among all unbiased estimators of the true mean, a fact critically exploited in the proof. 
%\textcolor{blue}{MH: Single arm CT Multi-armed bandits is looking odd. The bandit has single as well as multiple arm!}

Thus, as a function of $\mu$, the proposed algorithm has the optimal regret, while there is a logarithmic gap in terms of $T/\lambda$.\footnote{The ratio $T/\lambda$ is an invariant of the problem. See Remark \ref{rem:invariant}.} The result has an intuitive appeal since as $\mu < 1$ decreases, the regret increases, since for the CTMAB,  $\mu$ has to be estimated, and that becomes harder as $\mu$ decreases.

%AISTATSXXX

2. For the general CTMAB with multiple arms,  we propose an algorithm whose regret is at most $O\left(\frac{ \mu[1]K\log^2 (T/\lambda)}{\Delta^2}\right)$ when $\frac{K}{\mu[1]} = O\left(\frac{\mu[1]}{\Delta^2}\right)$, which is a practically reasonable regime. The derived results holds more generally and not just for $\frac{K}{\mu[1]} = O\left(\frac{\mu[1]}{\Delta^2}\right)$.
%AISTATS$$O\left(\max\left\{\left(\frac{\mu[1]K\log (T/\lambda)}{\Delta^2 }\right), \left(\frac{K^2}{\mu[1]} \log^2 (T/\lambda)\right)\right\}\right).$$  the regret is at most  
In converse, we show that for any online algorithm  its regret  is  $\Omega\left( \frac{K\mu[1]}{\Delta^2}\right)$ when $\Delta$ is small. 
%Exact result is presented in Theorem \ref{thm:lbmultiplearms}. 
To derive this lower bound we do not need the assumption that an online algorithm uses only unbiased estimators of $\mu[k], k=1, \dots, K$ for making its decisions since we are able to exploit the fact that there are multiple arms and an algorithm has to identify the best arm.
Similar to the single arm case, 
 as a function of $\mu[1]$ and $\Delta$, the proposed algorithm has the optimal regret, while there is a logarithmic gap in terms of $T/\lambda$. 
 %\textcolor{blue}{MH: As multi arm CTMAB is more general than single arm CTMAB, how is it that the proof of lower bound required lesser assumptions.}
%Compared to DMAB, where the optimal regret is $\Theta\left(K \log T / \Delta^2\right)$, with CTMAB we have an extra order $(\mu[1] \log T)$ term in the regret. 
%The gap of  
%order $\log(1/\Delta^2)$ is a by-product of the same order-wise gap between the best known algorithm and the lower bound for indentifying the best arm \cite{prism}.

%The interpretation of the result in this case is similar to the single arm case, and the optimal regret (neglecting $\log(1/\Delta^2)$ term) is $\Theta\left(K\frac{\mu[1] (\log(T))^2}{\Delta^2}\right)$, where the additional factor of $1/\Delta^2$ results because of the additional requirement of identification of the best arm. 
%Since $\Delta < \mu[1]$, the regret with multiple arms is larger than the single arm case. 
%This result re-emphasizes the fact that the CTMAB  is fundamentally different than the DMAB problem, where the optimal regret depends on the time horizon $T$, and gaps between the means of arms, and not on the means of the arms themselves.

%All the results of the paper easily extend to any convex functions $f$ for capturing sampling cost, instead of just $f(x) = 1/x$ the case considered for analysis, as pointed out in Remark \ref{rem:fconvex}. 

\subsection{Related Works}
%\vspace{-0.1in}
In prior work, various cost models have been considered for the bandit learning problems. The cost could be related to the consumption of limited resources, operational, or  quality of information required.

\noindent
{\bf Cost of resources:} In many applications (e.g., routing, scheduling) resource could be consumed as actions are applied. Various models have been explored to study learning under limited resources or cost constraints. The authors in \cite{JACM2018_BanditsWithKapsack} introduce {\it Bandits with Knapsack}  that combines online learning with integer programming for learning under constraints. This setting has been extended to various other settings like linear contextual bandits 
\cite{NIPS2017_LinearContextaulBanditsWithKanpsacks}, combinatorial semi-bandits \cite{AISTATS2018_CombibatorialSemiBandits}, adversarial setting \cite{FOCS2019_AdversarialBandits}, cascading bandits \cite{IJCAI2018_CasadingBandits}. The authors in \cite{Sigmetrics2015_BanditsWithBudgets} establish lower bound for budgeted bandits and develop algorithms with matching upper bounds. The case where the cost is not fixed but can vary is studied in \cite{AAAI2013_MABWithBdgetContraints}. 

\noindent
{\bf Switching Cost:} Another set of works study {\it Bandit with Switching Costs} where cost is incurred when learner switches from one arm to another arm \cite{STOC2014_DekelJian,NIPS2013_SwitchingCosts}. The extension to the case where partial information about the arms is available through feedback graph is studied in \cite{NIPS2019_SwitchingCostsFeedbackGraph}. For a detailed survey on bandits with switching cost we refer to \cite{DeEconomist2004_Juan}.

\noindent
{\bf Information cost:} In many applications the quality of information acquired depends on the associated costs (e.g., crowd-sourcing, advertising). While there is no bound on the cost incurred in these settings, the goal is to learn optimal action incurring minimum cost. \cite{ICML2015_CheapBandits,ICAASP2015_CostEffectiveSpectral} trade-offs  cost and information in linear bandits exploiting the smoothness properties of the rewards. Several works consider the problem of arm selection in online settings (e.g., \cite{AISTATS13_trapeznikov2013supervised,ICML14_seldin2014prediction})
involving costs in acquiring labels \cite{NIPS13_zolghadr2013online}.

Variants of bandits problems where rewards of arm are delay-dependent are studied in \cite{cella2020stochastic, pike2019recovering, kleinberg2018recharging}. In these works, the mean reward of each  arm is characterized as some unknown function of time. These setups differ from the CTMAB problem considered in this paper, as they deal with discrete time setup, and do not capture the cost associated with sampling rate of arms. Rested and restless bandit setups \cite{whittle1988restless} consider that distribution of each arm changes in each round or when it is played, but do not assign any penalty on rate of sampling.

In this work, our cost accounting is different from the above referenced prior work. The cost is related to how frequently the information/reward is collected. Higher the frequency, higher is the cost. Also, unlike the DMAB problem, there is no limit on the number of samples collected in a given time interval, however, increasing the sampling frequency also increases the cost. 

A multi-arm bandit problem, where pulling an arm excludes the pulling of any arm in future for a random amount of time (called delay) similar to our inter-sampling time has been considered in \cite{grygor}. However, in \cite{grygor} the delay experienced (inter-sampling time) is an exogenous random variable, while  it is {\bf a decision variable} in our setup. Moreover, the problem considered in \cite{grygor} is trivial with a single arm  similar to the usual DMAB, while it is non-trivial in our case as accuracy of the mean estimates play a crucial role.
%\vspace{-0.2in}

\section{The Model}
\label{sec:Model}
There are a total of $K$ arms and the total time horizon is $T$. At any time $t \in [0,T]$, any one of the arms can be played/sampled. On sampling arm $i$ at any time $t$, a 
random binary reward $X_i$ is obtained which follows a Bernoulli distribution with mean $\mu[i]$. We consider Bernoulli distribution here, however, all results will hold for bounded distributions.\footnote{All we need is that the considered concentration inequalities should hold.} If the time difference between any two consecutive samples is $\dt$, then the sampling cost for interval $\dt$ is $f(\dt) = 1/ \dt$. We make this choice for $f$ to keep the exposition simple, and more general convex functions can be analysed similarly, see Remark \ref{rem:fconvex}. The learning algorithm is aware of $T$. More discussion on this assumption is provided in Remark \ref{rem:horizon}. The ordered arms are denoted by $\mu[1] > \cdots > \mu[K]$, where $\mu[1]<1$.

%We consider a Multi-Armed Bandit (MAB) problem with $K$ arms where each arm has sampling cost that depends on the time taken to observe the sample. 

Let the consecutive instants at which any arm is sampled by a learning algorithm, denoted as $\cA$, be $t_0, t_1, t_2, \dots,$ where $t_0=0$, and the inter-sampling time be $\dt_i = t_{i}-t_{i-1}$. 
Let $k(t_i)$ denote the arm sampled at time $t_i$. Then the instantaneous expected payoff of $\cA$ from the $i^{th}$ sample is given as 
$\sfp_i = \EE{X_{k(t_i)}} -  \lambda f(\dt_i) = \mu_{k(t_i)} - \frac{\lambda}{\dt_i}$,
where  $\lambda$ is the trade-off parameter between the 
sampling cost and the reward. The cumulative expected payoff of the algorithm $\cA$ is given by 
\vspace{-0.1in}
\begin{equation}\label{eq:payoffA}
P_\cA(T) = \sum_{i=1}^{N_T}\sfp_i,
\end{equation}
where $N_T$ is the total number of samples obtained by $\cA$  over the horizon $T$.  Whenever necessary we also write $P_\cA(T)$ as $P_\cA([0,T])$ to specify the interval over which the payoff is being computed.
%Instead of just using $\sfp_i$ in \eqref{eq:payoffA}, $\max\{\sfp_i,0\}$ is included to 
%avoid pathological cases where with small probability the payoff can be an unbounded large negative quantity because of the presence of the sampling cost. This will also ensure that the maximum regret of any algorithm will be $\mu^2 T/4\lambda$ as pointed out in Remark \ref{rem:trivialubregret}.

The oracle policy that knows the mean values $\mu[k], k=1, \dots, K,$ always samples the best arm $\mu[1]$. 
%For the oracle policy, let the total number of samples   
%obtained over the time horizon $T$ be $N^o_T$, where the $j^{th}, 0\le j\le N^o_T-1$ sample is obtained at time $t^o_j$, with $\Delta t^o_i = t^o_{i}-t^o_{i-1}$ for $1\le i\le N^o_T$, and $t_0^o= 0$. Then the optimal cumulative payoff obtained by the oracle policy over the horizon $T$ is given as
%$P^\star (T)=\max_{i \in [K]} R^\star (\mu_i, T)$ where
%\vspace{-0.2in}
%\begin{align}
%&P^\star(T)=\max_{(\Delta t^o_i)_{i=1}^{N^o_T} } \mu[1] N^o_T - \sum_{i=1}^{N^o_T}\frac{\lambda}{\dt_i} .%\mbox{  and  } \sum_{j=1}^{N}\Delta t_j \leq T.
%\end{align}
%such that $\sum_{j=1}^{N^o_T} \Delta t^o_i \leq T $.
%A simple property of the cost function $\frac{1}{\dt_i}$ described next,  follows from the convexity of $1/x$.
\begin{prop}\label{prop:uniformsampling}
If $N$ samples are obtained in time $[0,t]$ at times $t_i$ with $\dt_i=t_i-t_{i-1}$, then the cumulative sampling cost $\sum_{i=1}^{N}\frac{1}{\dt_i}$ over time horizon $[0,t]$, where 
$\sum_{j=1}^{N} \Delta t_i \leq t,$ is minimized if the $N$ samples are obtained at equal intervals in $[0,t]$ for any $t$, i.e., $\Delta t_i = T/N$ $\forall \ i$.
\end{prop}
Proof of Proposition \ref{prop:uniformsampling} is immediate by noticing that $1/x$ is a convex function, and the fact that for a convex function $f$,  $x^\star_i = 1/n, \ \forall \ i=1,\dots, n$ is the optimal solution to 
$\min_{x_i}\sum_{i=1}^n f(x_i), \ \text{such that} $ $\  x_i \ge 0$, and $ \sum_{i=1}^n x_i \le 1.$
Using Proposition \ref{prop:uniformsampling}, we have that the payoff of the oracle policy is 
\begin{align}\label{eq:payofforaclebo}
P^\star(T) &= \max_{N^o_T}  \mu[1] N^o_T - \frac{N^o_T \lambda }{T/N^o_T} =  \max_{N^o_T}  \mu[1] N^o_T - \frac{(N^o_T)^2 \lambda }{T},
\end{align}
where $N^o_T$ samples are obtained in total by the oracle policy. 
Directly optimizing \eqref{eq:payofforaclebo} over $N^o_T$, we obtain that the optimal number of samples  obtained by  the oracle policy and the corresponding optimal payoff is given by Proposition \ref{prop:oracle}, assuming $\frac{\mu[1]T}{2\lambda}$ to be an integer. \footnote{If $\frac{\mu[1]T}{2\lambda}$ not an integer, then we check whether its floor or ceiling is optimal and use that as the value of $N^\star_T$.}
\begin{prop}\label{prop:oracle}
The oracle policy always samples arm $1$, $N^\star_T = \frac{\mu[1]T}{2\lambda}$ times in time horizon $[0,T]$ at equal intervals, i.e., at uniform frequency of $N^\star_T/T$. With $N^\star_T = \frac{\mu[1]T}{2\lambda}$, the optimal payoff \eqref{eq:payofforaclebo} is given by $P^\star(T)=\frac{\mu[1]^2T}{4\lambda}.$
\end{prop}

Note that the sampling frequency $N^\star_T/T$ of the oracle policy depends on the mean of the best arm, which distinguishes the CTMAB from the well studied DMAB.

\begin{remark}\label{rem:invariant}For fixed $\mu[1], \dots, \mu[K]$, CTMAB problem with parameters  $(\lambda, T)$ is equivalent to CTMAB problem with parameters $(c\lambda, cT)$ where $c > 0$ is a constant. To see this, if $\Delta t_i$ is the sampling duration with parameters  $(\lambda, T)$, then using $c\Delta t_i$ as the sampling duration with parameters  $(c\lambda, cT)$ results in the same payoff. Thus, $\frac{T}{\lambda}$ is an invariant of the considered problem, and for notational simplicity from here on we just write $T$ to mean $\frac{T}{\lambda}$.
\end{remark}

The {\bf regret} for an algorithm $\cA$ is defined as 
\begin{equation}\label{defn:regret}
\Regret_A(T) = P^\star(T)- P_\cA(T), %= \sum_{s=1}^{S(T)} \left(\mu_a - \lambda C(\Delta_s)\right),
\end{equation}
and the {\bf objective} of the algorithm is to minimize $\Regret_A(T)$. 
%Since $P^\star(T) = \frac{\mu[1]^2T}{4}$, we characterize the regret of any algorithm as 
%$\Theta(\mu[1]^2T^p)$ for some $p<1$ that could depend on parameters of the problem $\mu[i], \Delta, T,$ etc.
%\begin{remark}\label{rem:trivialubregret}
%Note that because of $\max\{\sfp_i,0\}$ in the payoff definition \eqref{eq:payoffA}, the regret of any algorithm can be at most 
%$\frac{\mu[1]^2T}{4}$. Without the $\max\{.\}$, because of the presence of the sampling cost with the negative sign, the regret can be unboundedly large for a given algorithm with some small enough probability which could be hard to bound.  Also, this is another important difference between the usual DMAB problem where the regret of an arbitrary algorithm is at most $(\mu[1]-\mu[K])T$.
%\end{remark}
We begin our discussion on the CTMAB problem by considering the case when there is only a single arm, which as discussed before is a non-trivial problem.

\section{CTMAB with A Single Arm}
\label{sec:SingleArm}
In this section, we consider the CTMAB,  when there is only a single arm with true mean $\mu$, and $\mu<1$. Results when $\mu > 1$ can be obtained by using appropriate scaling similar to the usual DMAB problem. With the single arm, we denote the binary random reward obtained by sampling at time $t_i$ as $X_i$, and $\bbE\{X_i\} = \mu, \ \forall \ i$.  %First, we propose an algorithm and upper bound its regret. 
%Next, we derive a lower bound for any online algorithm for which Assumption \ref{rem:unbiased} holds. The upper and lower bounds are shown to differ only in logarithmic terms.

%\vspace{-.2in}
\subsection{Algorithm CTSAB}
%\vspace{-.1in}
In this section, we propose an algorithm that achieves a regret within logarithmic terms of the lower bound derived in Theorem \ref{thm:lbsinglearm}.

%\end{minipage}
{\bf Algorithm CTSAB:} Divide the total time horizon $[0,T]$ in two periods: {\bf learning} and {\bf exploit}. 
Pick $0 < \epsilon <1 $. 
%The choice of $\epsilon$ will determine the speed of the algorithm, and the regret guarantee. The smaller the $\epsilon$, the better is the regret but slower the speed.
%In Remark \ref{rem:choiceeps}, we discuss one possible choice of $\epsilon$ if an upper bound $\mu_{\max}$ is known for the actual value of $\mu \le \mu_{\max}$.
The algorithm works in phases, where phase $1$ starts at time $0$ and ends at time $T^{\epsilon}$.  Subsequently, phase $i$,  $2\le i\le i^\star$ ($i^\star$ is defined in \eqref{istaddef}) starts at time $T^{(i-1)\epsilon}$ and ends at $T^{i\epsilon}$ with duration $T^{i \epsilon} - T^{(i-1)\epsilon}$. %We will choose $\epsilon$ small enough to ensure that $\epsilon < p$. 
For each phase $i$, $1\le i \le i^\star$, the algorithm obtains $N_i = \kappa \log (T) T^{2/3i\epsilon}$ samples
in phase $i$ equally spaced in time, i.e., at uniform frequency in that phase. 
At the end of phase $i$, the total number of samples obtained is $N^i = \sum_{j\le i}N_j$, 
and let 
\begin{equation}\label{eq:empest1}
{\hat \mu}_i = \frac{1}{N^i} \sum_{k=1}^{N^i} X_k,
\end{equation} be the empirical average of all the sample rewards
obtained until the end of phase $i$. 

The absolute difference between the empirical average and the true mean is defined as \text{err}. 
\begin{remark}
With abuse of notation, we interchangeably use $\text{err}_N, \text{err}_i, \text{err}_t$ to denote the error after $N$ samples or after phase $i$ or at time $t$.
Thus, the error in estimating $\mu$ at the end of phase $i$ is 
$\text{err}_i = |{\hat \mu}_i -\mu|$.
\end{remark}

We next define $i^\star$, and the algorithm to follow after phase $i^\star$. 
For a given $\delta$ (input to the algorithm), let $i^\star$ be the earliest phase at which 
\begin{equation}\label{istaddef}
\sqrt{\frac{\log(2/\delta)}{N^{i^\star}}} < 
\frac{{\hat \mu}_{i^\star}}{2},
\end{equation} where $N^{i^\star} = \sum_{j\le i^\star}N_j$. If no such $i^\star$ is found, then we define that the algorithm $\textsf{fails}$.

 The {\bf learning} period ends at phase $i^\star$, and the {\bf exploit} period starts from the next phase $i^\star+1$.
 Each phase $i > i^\star$ is of the same time duration $ T^{i^\star \epsilon}$ till the total 
time horizon $T$ is reached.
Starting from phase $i^\star+1$ and for 
all subsequent phases $k \ge i^\star+1$, 
the algorithm assumes ${\hat \mu}_{k-1}$ \eqref{eq:empest1} to be 
the true value of $\mu$, and obtains $N_{k} = \frac{{\hat \mu_{k-1}}T^{i^\star \epsilon}}{2}$ samples in phase $k$, equally spaced in time, and ${\hat \mu_{k}}$ is updated at the end of each phase $k \ge i^\star+1$ using all the samples obtained so far since time $t=0$. 
%Time till the end of phase $i^\star$ is called the  {\it learning period}.
The pseudo code for the algorithm is given in Algorithm \ref{alg:CTBandit} (presented in supplementary material).

The proposed algorithm CTSAB follows the usual approach of exploration and exploitation, however, there are {\bf two non-trivial problems} being addressed, whose high level idea is as follows.
The aim of the learning period is to obtain sufficient number of samples $N$, such that $\text{err}_N < \mu$.  Since otherwise, the payoff obtained in phases  after the learning period cannot be guaranteed to be positive, following Lemma \ref{lem:tptime}. 
So the first problem is a {\bf stopping problem}, checking for $\text{err}_N < \mu$, which is non-trivial, since $\mu$ is unknown. For this purpose, 
 a surrogate condition  \eqref{istaddef}  is defined, and the learning period is terminated as soon as \eqref{istaddef} is satisfied for a particular choice of $\delta$. Choosing 
 $\delta = \frac{1}{T^2}$, using Corollary \ref{cor:conc} and Lemma \ref {lem:direct}, we show that  whenever \eqref{istaddef} is satisfied, $\text{err}_N < \mu$ with probability at least $1-\delta$. 
 
 %We choose $\delta = \frac{1}{T^2}$, and show in the proof of Theorem  \ref{thm:ubsinglearm}, that  \eqref{istaddef} is satisfied for $i^\star = \frac{p^\star}{\epsilon}+1$ with probability at least $1-\delta$.
 %We also 
% $\sqrt{\frac{\log(2/\delta)}{N}} < \frac{{\hat \mu}_N}{2}$ with probability at least $1-\delta$, which implies that $\text{err}_N < \mu$ with probability at least $1-\delta$. 

The second problem remaining is to bound the time by which the learning period ends, i.e., \eqref{istaddef} is satisfied. 
We need this bound since non-zero payoff can be guaranteed only for phases that belong to the exploit period that starts after the learning period. Towards that end, we show that the length of learning period is  $O(T^{p^\star+\epsilon})$ with high probability in Lemma \ref{lem:ublearningperiod}, where $p^\star$ is defined as follows.
 \begin{equation}\label{defn:p}
T^{p^\star}=  \frac{1}{\mu^3 }  \ \text{for} \ 0\le p^\star \le 1.
 %\right\}.
%p = \min \left\{k\ge 0:\frac{c(1-2/T^\alpha)^2}{\mu^4 T^{k+\alpha}} = 
%O\left(\mu^2 T^k\right)\right\}\ \forall \ \alpha>0, \ \text{for} \ c = \lambda\left(\frac{4} {2 \ln(2)}\right)^2.
\end{equation}
 This automatically means that we are assuming that the time horizon $T$ is at least as large as  $\frac{1}{\mu^3}$. The lower bound on regret in Theorem \ref{thm:lbsinglearm} implies that $T=o( \frac{1}{\mu^3})$ is a degenerate regime for the studied problem. 
The main result of this subsection is as follows. 
\begin{theorem}\label{thm:ubsinglearm} The expected regret of algorithm \textsc{CTSAB} while choosing $\delta= \frac{1}{T^2}$ is 
$$ O( \mu^2 T^{p^\star+ (4/3)\epsilon}\log(2 T^2)\log(2 T))(1-\frac{1}{T^{1+\epsilon}})+ \frac{\mu^2}{4}\left( \frac{1}{T^\epsilon} + \frac{1}{T}\right),$$ for any $\epsilon>0$, where $p^\star$ as defined in \eqref{defn:p}.
\end{theorem}
All missing proofs can be found in the supplementary material.
Given that $T$ is fixed and $\epsilon>0$ is a variable, we choose $\epsilon$ such that $T^\epsilon$ is a constant. 
With this choice, since $\mu <1 $, using \eqref{defn:p} the regret bound of the CTSAB algorithm (Theorem \ref{thm:ubsinglearm}) is $O\left(\frac{(\log(T)^2}{\mu}\right)$, that differs from the lower bound to be derived in Theorem 
\ref{thm:lbsinglearm} only by logarithmic terms.

%In the proof of Theorem \ref{thm:ubsinglearm}, we show that the  learning period ends in at most  $O(T^{p^\star+\epsilon})$ time in Lemma \ref{lem:ublearningperiod} for which we count zero payoff for the algorithm. 
%Thus, the regret of the CTSAB algorithm  in the learning period is at most $O( \mu^2 T^{p^\star+\epsilon})$. To complete the proof, we show that the payoff obtained in the remaining time defined as the exploit period differs from that of the oracle policy by only constant terms. 

%\vspace{-0.15in}
\subsection{Lower Bound}
For deriving the lower bound in the single arm case, we consider only those algorithms that satisfy the following assumption.
\begin{assumption}\label{rem:unbiased}
For any online algorithm for the single arm CTMAB problem maximizing \eqref{eq:payoffA}, the decision variables are the sampling times $t_i'$s.
We consider only those online algorithms that at any time make these decisions depending on an unbiased 
estimate of $\mu$.
\end{assumption}

\begin{theorem}\label{thm:lbsinglearm}
Let $\mu < 1/4$. Let the regret of any online algorithm for the single arm CTMAB for which Assumption \ref{rem:unbiased} holds be 
$g(\mu, T)$. Let $g(\mu, T)$ be expressed as $g(\mu, T) = \mu^2 T^p$ for some $p<1$.\footnote{This specific structure does not limit the generality of all possible regret functions, and is only being considered for the simplicity of analysis.}
Then $p$ must satisfy, 
$p \ge \min\{p_1, p_2\},$ where
\vspace{-0.1in}
\begin{equation}\label{defn:p1}
T^{p_1} \ge  \frac{\sfc_1}{\mu^3} \ \  \text{for}\ \ 0\le p_1\le 1,
%p_1= \min\left\{k\ge 0:  
  %\right\},
  \quad 
  \text{and}
\end{equation}  
\vspace{-0.1in}
\begin{equation}\label{defn:p2}
p_2=  \min \left\{k\ge 0:\frac{1}{T^{k}}\left(\frac{\sfc_2}{\mu^2} \right)^2 = 
\max\left\{\mu^2 T^k, \frac{\sfc_3}{\mu}\right\} \right\},
\end{equation}
where 
$\sfc_1, \sfc_2, \sfc_3$ are constants.
\end{theorem}
Thus,  $p_1$ and $p_2$ satisfy $T^{p_1} =  \Omega\left(\frac{1}{\mu^3 }\right)$ and $T^{p_2} =  \Omega\left(\frac{1}{\mu^3 }\right)$, and the regret of any online algorithm satisfying Assumption \ref{rem:unbiased} is $g(\mu, T) = \mu^2 T^p = \Omega\left(\frac{1}{\mu}\right)$.

%The precise result is given in Theorem \ref{thm:lbsinglearm}, which loosely translates to a regret lower bound of $\Omega(\mu^2 T^p)$, where $p$ is such that $\mu^3 T^p$ is a constant. Thus, the effective lower bound on regret is $\Omega(1/\mu)$.
The main idea to prove Theorem \ref{thm:lbsinglearm} is detailed in Appendix \ref{app:thm:lbsinglearm:idea}.
Comparing Theorem \ref{thm:ubsinglearm} and Theorem \ref{thm:lbsinglearm}, we see that algorithm \textsc{CTSAB} achieves optimal regret up to  logarithmic factors of $T$.

\section{CTMAB with Multiple Arms}
\label{sec:MultipleArms}
In this section, we consider the general CTMAB problem with $K$ arms that have means $1> \mu[1] > \cdots >\mu[K]$ and $\Delta = \mu[1]-\mu[2] >0$, and the objective is to minimize the regret defined in \eqref{defn:regret}. 
%We first present an algorithm called CTMAB and upper bound its regret. Subsequently, we derive a lower bound on the regret of any algorithm.
\subsection{Upper Bound - Algorithm \textsc{CTMAB}}\label{sec:ubmulti}
We propose an algorithm for the CTMAB problem, called the \textsc{CTMAB} algorithm, that is neither aware of $\Delta$ or the actual means $\mu[k]$, and show that its regret is within logarithmic terms of the lower bound (Theorem \ref{thm:lbmultiplearms}). 
The first part of the algorithm is called the {\it estimation period}, that is designed to estimate the mean $\mu[1]$ of the best arm within an error of at most $\mu[1]$ with high probability. The estimation period of algorithm \textsc{CTMAB} is equivalent to the learning period of algorithm \textsc{CTSAB} applied simultaneously to all the $K$ arms. 

Similar to algorithm \textsc{CTSAB}'s learning period, the estimation period of algorithm \textsc{CTMAB} ends as soon as the condition 
$\sqrt{\frac{\log(2/\delta)}{N^t[k]}} < \frac{{\hat \mu}_t[k]}{2}$ is satisfied for some arm $k \in \{1, \dots, K\}$, where $N^t[k]$ is the number of samples obtained for arm $k$ 
by time $t$, and 
${\hat \mu}_t[k]$ is the empirical average of the sample rewards for arm $k$ with  $N^t[k]$ samples. If this condition is not satisfied at all, we define that the algorithm {\it fails}. Note that the estimation period is not trying to identify the best arm, but the objective is to just estimate the mean of the best arm within an error of $\mu[1]$. In particular, we will show that with high probability whenever the algorithm does not fail, $\frac{2}{3}\mu[1]\le {\hat \mu}_t[a_t] \le 2\mu[1]$, where ${\hat \mu}_t[a_t]$ is the estimate of $\mu[1]$ output by the algorithm.

Once the estimation period is over, the {\it identification period} begins that is used to identify the best arm with high probability. Using the estimate ${\hat \mu}[1]$ of the mean of the best arm $\mu[1]$ found in the estimation period,  in the identification period, the best arm is identified using the \textsc{LUCB1} algorithm \cite{ICML12_PACSubsetMAB}, where {\bf samples are obtained at speed one sample per $1/{\hat \mu}[1]$ time}. The speed choice is required to be dependent on $\mu[1]$ to keep the regret low, and that is why we need the estimation 
period to get a `good' estimate of $\mu[1]$. Once the identification period ends (where the best arm is identified with high probability), the final period called {\it exploit} begins that only considers the arm identified in the identification period as the best, and executes the exploit period of algorithm \textsc{CTSAB}. The pseudo code for algorithm \textsc{CTMAB} is provided in Algorithm \ref{alg:CTMultiBandit} (presented in the supplementary material).

%\begin{remark}
%For the continuous time CTMAB problem, we use the \textsc{Track and Stop} algorithm \cite{garivier2016optimal}, that is designed for discrete setting, by obtaining
 %the number of samples suggested by \textsc{Track and Stop} algorithm to identify the best arm are obtained at frequency one sample per $1/{\hat \mu}[1]$ time.  
%\end{remark}

%We next define a quantity that is useful for deriving an upper bound on the regret of the \textsc{CTMAB} algorithm.

%\begin{equation}\label{defn:pm}p_m^\star = \min\left\{r\ge 0 : \frac{K}{\Delta^2 \mu[1] T^r} = 1 \right\} .\end{equation}
%Note that $p_m^\star \le p_2^m$ \eqref{defn:pm2}.

The main result of this subsection is as follows.
\begin{theorem}\label{thm:ubmultiarm} With the choice of $\delta=\frac{1}{T^2}$, the expected regret of the \textsc{CTMAB} algorithm is at most
\begin{align*}\label{}
  & O\left(\max\left\{\frac{\mu[1]K}{\Delta^2}\log T, \frac{K^2}{\mu[1]}\log^2 T, \frac{\log^2 T}{\mu[1]} \right\}\right)\\
  & \cdot \left(1-\frac{1}{T^{1+\nu_m}}\right) \left(1-\frac{1}{T^2}\right)\\
	&  \quad + O(T)\left(1-\left(1-\frac{1}{T^{1+\nu_m}}\right) \left(1-\frac{1}{T^2}\right)\right)
  %  \max\left\{O\left(\frac{\mu[1]K\log T}{\Delta^2 }\right), O\left(\mu[1]K^2 \log^2 (T)\right), O\left(\frac{K}{\Delta^2 \mu[1]}\log T\right), O\left(\frac{ K^2}{\mu[1]} \log^2T\right) \right\}\\
%	& \quad \times\left(1-\frac{T}{T^{\nu_m}}\frac{1}{T^2}\right) \left(1-\frac{1}{T^2}\right)  + \mu[1]^2T/4\left(1-\left(1-\frac{T}{T^{\nu_m}}\frac{1}{T^2}\right) \left(1-\frac{1}{T^2}\right)\right),
  %\max\left\{O\left(\frac{\mu[1]K\log T}{\Delta^2 \lambda}\right), O\left(\frac{\lambda K^2}{\mu[1]} \log^2 (T)\right), O\left(\frac{\lambda \mu[1]^2K\log T}{\Delta^2}\right)\right\} \\
  %&  \quad\quad \quad\left(1-\frac{T}{T^{\nu_m}}\frac{1}{T^2}\right) \left(1-\frac{1}{T^2}\right) +o(T),\\
\end{align*}
%$$O\left(\mu[1]^2T^{p^\star_m} \log (T^2)\log (T)\right)\left(1-\frac{T}{T^{\nu_m}}\frac{1}{T^2}\right) \left(1-\frac{1}{T^2}\right) +o(T)$$
 where $T^{\nu_m}$ is the width of each phase after the identification period is over.
\end{theorem}
With $\frac{K}{\mu[1]} = O(\frac{\mu[1]}{\Delta^2})$ (which is a reasonable setting since $K$ is typically not too large),  the regret of the \textsc{CTMAB} algorithm is $O\left(\frac{K\mu[1](\log (T))^2}{\Delta^2}\right)$ matching the lower bound to be derived in Theorem \ref{thm:lbmultiplearms} upto logarithmic terms. 
%Thus upto $\log(1/\Delta^2)$ term, the upper and lower bound on the regret, match. The gap of  
%order $\log(1/\Delta^2)$ is a by-product of the same order-wise gap between the best known algorithm and the lower bound for indentifying the best arm \cite{prism}.

The basic idea to derive Theorem \ref{thm:ubmultiarm} is as follows. In Lemma \ref{lem:ublearningperiod} we show  that for the truly best arm, arm $1$, the condition $\sqrt{\frac{\log(2/\delta)}{N^t[1]}} < \frac{{\hat \mu}_t[1]}{2}$ is 
satisfied by time at most $O(T^{p^\star + \epsilon})$ with probability at least $1-1/T^2$, where $T^\epsilon$ is a constant and $p^\star$ is as defined in \eqref{defn:p} with $\mu[1]=\mu$. Thus, the estimation phase terminates by time $O(T^{p^\star})$ with probability at least $1-1/T^2$. Moreover, we show in Lemma \ref{lem:estimationperiod}, that whenever the estimation phase terminates, the estimated mean ${\hat \mu}[a_t]$ satisfies $|{\hat \mu}[a_t] - \mu[1]|\le \frac{\mu[1]}{2}$ with high probability. 
%The time consumed by the estimation period is $O(T^{p^\star})$, and the number of samples obtained (that control the sampling cost) in the estimation period is at most $K(\kappa \log (T))T^{2/3 (i^\star+1) \epsilon }$, where $i^\star$ is the phase in which the estimation period terminates, and $i^\star \le p^\star/\epsilon+1$ (Lemma \ref{lem:ublearningperiod}), and $\kappa$ is a constant. 
Consequently, we show that the payoff  of the estimation period is at least $-O\left(\frac{K^2 \log^2 T}{\mu[1]}\right)$ with high probability as shown in Lemma \ref{lem:payoff1ctmab}.

%Next, in the identification period, the \textsc{Track and Stop} algorithm is executed, where the samples needed by it are obtained at frequency  one sample per 
%$1/({\hat \mu}[1])$ time. 
From \cite{ICML12_PACSubsetMAB}[Thm 6], we know that the \textsc{LUCB1} algorithm needs $O\left(\frac{K\log T}{\Delta^2}\right)$ expected samples to identify the 
best arm with high probability (setting $\epsilon=0$ and $\delta=1/T$). Thus, by obtaining these samples at a frequency  of one sample per 
$1/({\hat \mu}[1])$ time in the identification period, where $|{\hat \mu}[1]-\mu[1]| \le \mu[1]$ is guaranteed, the total time needed for the identification period is $O\left(\frac{K\log(T) }{\Delta^2\mu[1]}\right)$ with high probability.
The choice of the frequency of obtaining samples in the identification period needs to depend on $\mu[1]$ to keep the regret of the identification period small. With this choice of sampling frequency, the payoff of the identification period is at most $-O\left(\frac{4K}{\mu[1]\Delta^2} \log \left(T^2\right) \right)$ as shown in 
\eqref{eq:payoffmultiarm1} with high probability.

Once the identification period ends, algorithm \textsc{CTMAB} is identical to algorithm \textsc{CTSAB} where the single arm to consider is the arm identified as the best arm in the identification period. 
%Recall that in the analysis of Theorem \ref{thm:ubsinglearm} for algorithm \textsc{CTSAB}, the payoff of the exploit period is computed assuming that the error in estimating the single arm is at most $\mu$ before the exploit phase begins in algorithm \textsc{CTSAB}. 
 With the \textsc{CTMAB} algorithm, the number of samples $N^{[1]}$
obtained by the \textsc{LUCB1} algorithm  (for the best arm) in the identification period for the best identified arm is $\Omega(1/\Delta^2)$ which is more than $\Omega(1/\mu[1]^2)$, since $\Delta < \mu[1]$. Thus, using the single arm  case result, at the end of identification period of algorithm \textsc{CTMAB}, the error in estimating the mean of the best arm $\mu[1]$ is at most 
$\mu[1]$ with high probability. Therefore, we can directly use the payoff guarantee of \textsc{CTSAB}  during its exploit period to bound the payoff of 
\textsc{CTMAB}  during its exploit period.

\begin{figure*}[!h]
	\centering
	\subfloat[]{\label{fig:MultiArm1}
		\includegraphics[scale=0.28]{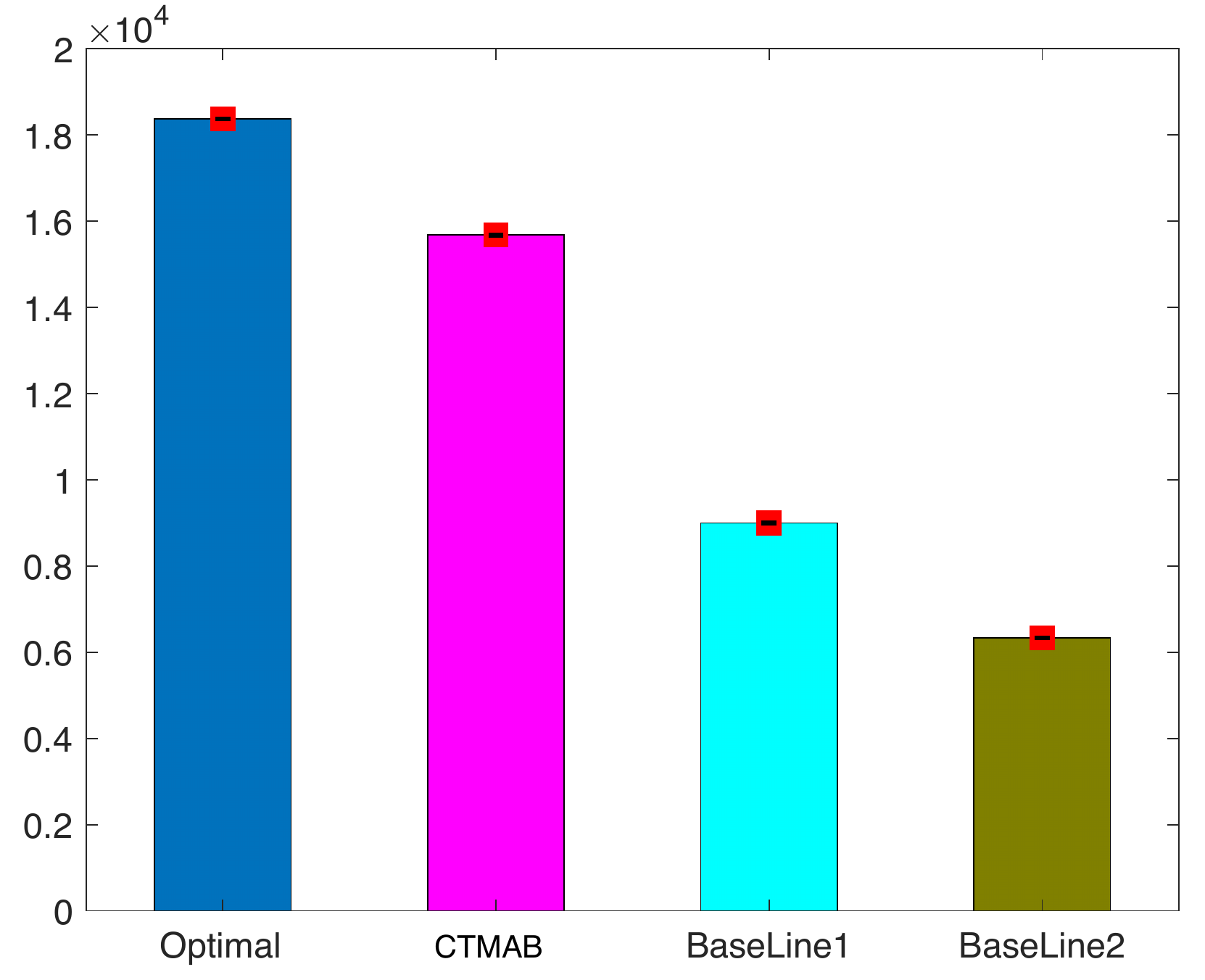}}
	\subfloat[]{\label{fig:MultiArm2}
		\includegraphics[scale=0.28]{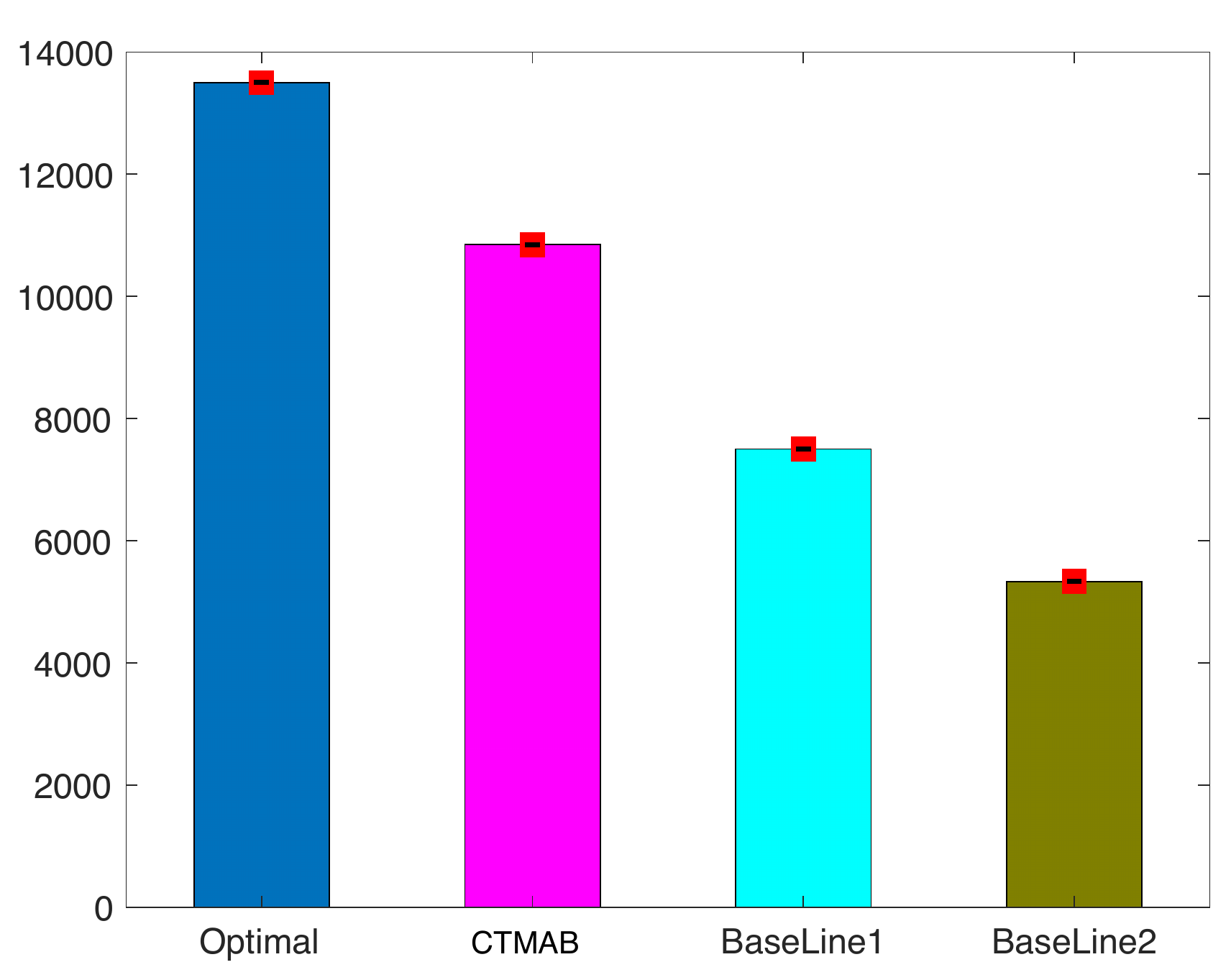}}
	\subfloat[]{\label{fig:MultiArm3}
		\includegraphics[scale=0.28]{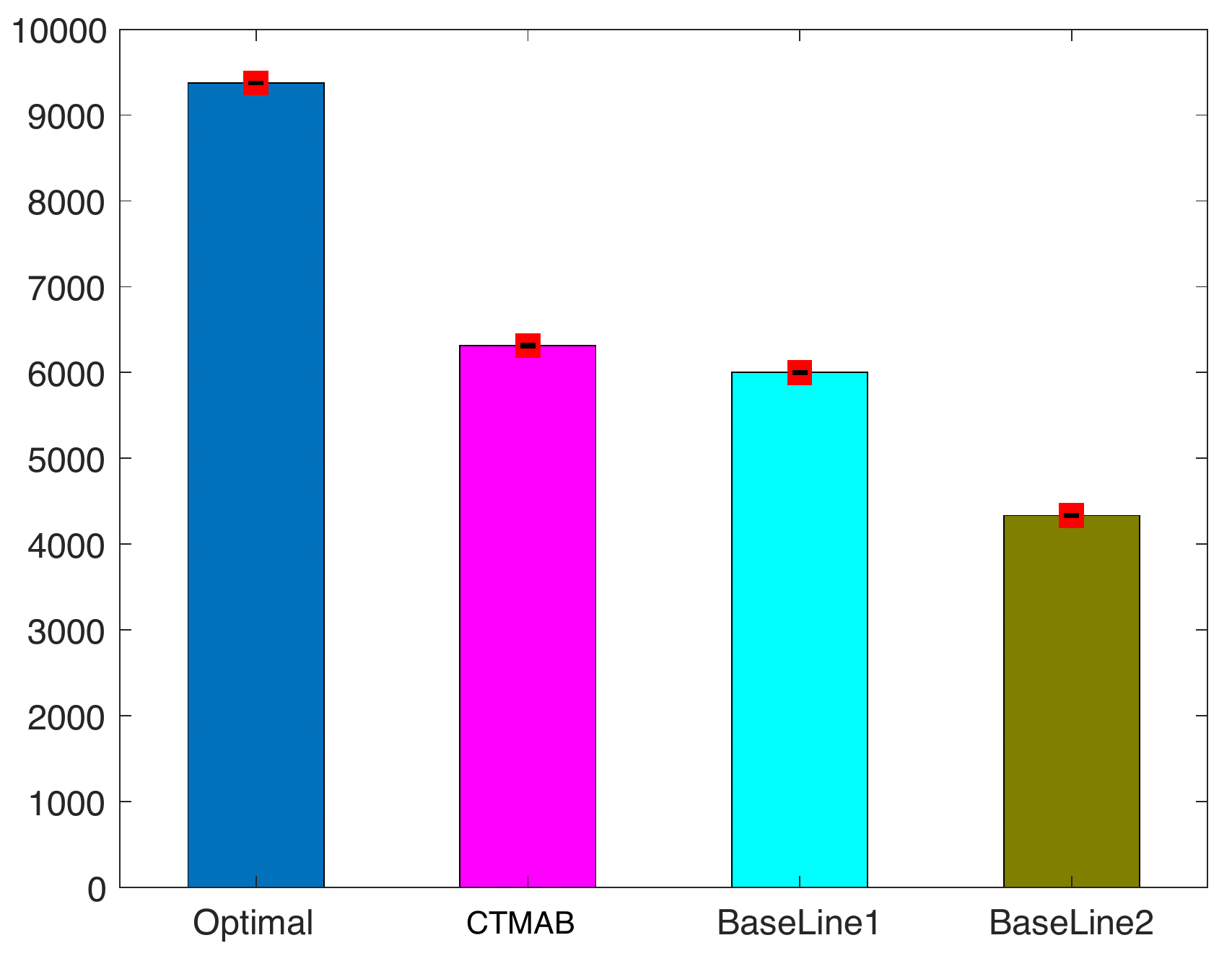}}
	\caption{Comparison of cumulative reward of algorithm {\textsc CTMAB} with other policies for different arm means. }
	\label{fig:Payoff Comparison}
\end{figure*}
%\vspace{-0.1in}
\subsection{Lower Bound}\label{sec:lbmulti}
%\vspace{-0.1in}
For the multiple arm CTMAB problem we derive a lower bound on the regret of any algorithm by exploiting the fact that the algorithm has to identify the best arm with a certain probability by a certain time. 
\begin{theorem}\label{thm:lbmultiplearms}
 Let the regret of any online algorithm for the multiple arm CTMAB  problem be 
$g_m(\mu[1], \dots, \mu[K], T, \Delta)$ which can be expressed as 
$g_m(\mu[1], \dots, \mu[K], T, \Delta) = \mu[1]^2 T^p_m,$ for some $p_m<1$. 
%\footnote{This specific structure does not limit the generality and is only being considered for the simplicity of analysis.}
Then $p_m \ge \min\{p_1^m, p_2^m\},$ where
\begin{equation}\label{defn:pm1}
 T^{p_1^m} = \frac{2  K \log\left(\frac{c_{12}}{8}\right)}{c_{12}(\Delta/2)^2 \mu[1]} \ \ \text{for} \ \ 0\le p_1^m \le 1, \quad \text{and}
\end{equation}
 $p_2^m  = \min \left\{k\ge 0: \frac{ \left(\frac{\sfc_4 K}{(\Delta/2)^2}
 	\log\left(\sfc_5\frac{c_{12}}{8}\right)\right)^2}{c_{12}T^{k}} \right.$
\begin{align} \label{defn:pm2}
	& \left. \quad \quad = 
\max\left\{\mu[1]^2 T^k, \mu[1]\frac{\sfc_4 K}{(\Delta/2)^2}
 	\log\left(\sfc_5\frac{c_{12}}{8}\right)\right\} \right\},
\end{align}
where $c_{12} = \frac{\mu[1]^2}{\mu[2]^2}$ and $\sfc_4,\sfc_5$ are constants.
\end{theorem}

When $\Delta$ is small, i.e., $\mu[1]\approx\mu[2]$, $p_1^m$ and $p_2^m$ satisfy  
$T^{p_1^m} =  \Omega(\frac{ K}{\mu[1]\Delta^2})$ and $T^{p_2^m} = \Omega(\frac{ K}{\mu[1]\Delta^2})$. 
Hence the regret of any online algorithm is $g_m(\mu[1], \dots, \mu[K], T, \Delta) = \Omega( K\mu[1]/\Delta^2)$ in the small $\Delta$ regime.

The main idea used to derive this lower bound is as follows. Let the regret of any algorithm $\cA$ be $\mu[1]^2T^{p_m}$. Then consider time $T^{p_m + \alpha}$ for any $\alpha>0$. We show that if the probability of correctly identifying the best arm with algorithm $\cA$ is less than $1-1/T^\alpha$ at time $T^{p_m + \alpha}$, 
then the regret of $\cA$ is $> \mu[1]^2T^{p_m+\alpha}$, contradicting the assertion that the regret of $\cA$ is $\mu[1]^2T^{p_m}$. Thus, the probability of identifying the best arm with $\cA$ at time $T^{p_m + \alpha}$ must be greater than $1-1/T^\alpha$. This necessary condition implies a lower bound on the number of samples $N_{T^{p_m+\alpha}}$  to be obtained by $\cA$ in time period $[0,T^{p_m + \alpha}]$ using Lemma \ref{lem:lbpureexplore}. 
%Therefore, we get a lower bound on the frequency of sampling ($N_{T^{p_m+\alpha}}/T^{p_m+\alpha}$) to be employed in time $[0,T^{p_m + \alpha}]$.
 Accounting for the sampling cost resulting out of this lower bound, gives us the  lower bound of  Theorem \ref{thm:lbmultiplearms}. 
 %Consider an algorithm $\cA$ and let $\delta_I=\bbP(k^\star \ne 1)$ be the probability that the algorithm 
%mis-identifies the best arm (at time $T^{p+\alpha}$ after obtaining $N$ samples 
%in time $[0,T^{p+\alpha}]$ for any $\alpha>0$. 

Comparing the upper bound on regret for algorithm CTMAB (Theorem \ref{thm:ubmultiarm}) and the lower bound (Theorem \ref{thm:lbmultiplearms}), we see that algorithm CTMAB achieves optimal regret up to  logarithmic factors of $T$.

\section{Numerical Results}
\label{sec:Exp}

In this section, we compare the performance of our algorithm against the oracle policy and a baseline policy that does not adapt to the estimates of the arm means. The baseline policy samples the optimal arm at a fixed  interval of $1/aT$, where $a>0$ is a constant that determines the rate of sampling. The payoff of the baseline policy over a period $T$ is $a T(\mu[1]-\lambda a)$, and  the payoff is positive and increasing for all $a \leq \frac{\mu[1]}{2\lambda}$ achieving maxima at $a= \frac{\mu[1]}{2\lambda}$.

We consider the multiple arms case, and evaluate the performance of the CTMAB algorithm with $K=5$ arms. %As the number of samples to collect prescribed by the Median Elimination Algorithm is overly pessimistic, instead of taking $\frac{1}{(\beta_l/2)^2}\log(3/\delta_l)$ number of samples for arms (that are not eliminated), we fix the number of samples at a constant value of $3000$ as done while executing pure explorations algorithms in practice. 
We simulate the CTMAB algorithm on three sets of mean vectors
${\boldsymbol \mu}=(\mu[1], \dots, \mu[5])$, with values $(0.35, 0.2, 0.15, 0.1, 0.08), (0.3, 0.2, 0.15, 0.1, 0.08),$ and $(0.25, 0.2, 0.15, 0.1, 0.08)$, and plot the cumulative payoff for the oracle policy and the CTMAB algorithms  in Fig. \ref{fig:Payoff Comparison}. The problem instances are chosen to have a decreasing sub-optimality gap and hence increasingly difficult to learn. As seen, the CTMAB performance is close to that of the oracle policy, and the regret degrades with reducing value of $\mu[1]$ and the sub-optimality gap. The results for the single arm case are provided in the supplementary material. 
%Further, the reward from the base-line policies are almost close to zero and hence not depicted in the figure.
%\begin{figure}
%	\centering
%	\includegraphics[scale=.5]{MultipleArm}
%	\label{fig:MultipleArm}
%	\caption{Payoff comparison of the oracle and the CTMAB algorithm on different problem instances, where $T=60000, \gamma=0.01, \lambda=0.1, \nu_m=0.1$.}
%\end{figure}

\label{sec:Conclusion}
%\vspace{-0.2in}
\section{Conclusions}
%\vspace{-0.2in}
In this paper, we have a introduced a new continuous time multi-arm bandit model (CTMAB), that is well motivated from applications in 
crowdsourcing and inventory management systems. The CTMAB is fundamentally different than the popular DMAB, and to 
the best of our knowledge has not been considered before. The distinguishing feature of the CTMAB  is that the oracle policy's decision depends on the mean  of the best arm, and this makes even the single arm problem non-trivial. To keep the model simple, we considered a simple sampling cost function, and derived almost tight upper and lower bounds on the 
optimal regret for any learning algorithm.  
\bibliographystyle{plainnat}
\bibliography{ref}

\begin{thebibliography}{29}
\providecommand{\natexlab}[1]{#1}
\providecommand{\url}[1]{\texttt{#1}}
\expandafter\ifx\csname urlstyle\endcsname\relax
  \providecommand{\doi}[1]{doi: #1}\else
  \providecommand{\doi}{doi: \begingroup \urlstyle{rm}\Url}\fi

\bibitem[Abinav and Slivkins(2018)]{AISTATS2018_CombibatorialSemiBandits}
Karthik Abinav and Sankararaman~Aleksandrs Slivkins.
\newblock Combinatorial semi-bandits with knapsacks.
\newblock In \emph{International Conference on Artificial Intelligence and
  Statistics (AISTATS)}, 2018.

\bibitem[Agrawal and
  Devanur(2016)]{NIPS2017_LinearContextaulBanditsWithKanpsacks}
Shipra Agrawal and Nikhil~R. Devanur.
\newblock Linear contextual bandits with knapsacks.
\newblock In \emph{Neural Information Processing Systems (NIPS 2016)}, 2016.

\bibitem[Arora et~al.(2019)Arora, Marinov, and
  Mohr]{NIPS2019_SwitchingCostsFeedbackGraph}
Raman Arora, Teodor~V. Marinov, and Mehryar Mohr.
\newblock Bandits with feedback graphs and switching costs.
\newblock In \emph{Advances in Neural Information Processing Systems(NIPS)},
  2019.

\bibitem[Badanidiyuru et~al.(2018)Badanidiyuru, Kleinberg, and
  Slivkins]{JACM2018_BanditsWithKapsack}
Ashwinkumar Badanidiyuru, Robert Kleinberg, and Aleksandrs Slivkins.
\newblock Bandits with knapsacks.
\newblock \emph{Journal of ACM}, \penalty0 (13), 2018.

\bibitem[Bubeck et~al.(2012)Bubeck, Cesa-Bianchi,
  et~al.]{NOW12_bubeck2012regret}
S{\'e}bastien Bubeck, Nicolo Cesa-Bianchi, et~al.
\newblock Regret analysis of stochastic and nonstochastic multi-armed bandit
  problems.
\newblock \emph{Foundations and Trends{\textregistered} in Machine Learning},
  5\penalty0 (1):\penalty0 1--122, 2012.

\bibitem[Cella and Cesa-Bianchi(2020)]{cella2020stochastic}
Leonardo Cella and Nicol{\`o} Cesa-Bianchi.
\newblock Stochastic bandits with delay-dependent payoffs.
\newblock In \emph{International Conference on Artificial Intelligence and
  Statistics}, pages 1168--1177, 2020.

\bibitem[Cesa-Bianchi et~al.(2013)Cesa-Bianchi, Dekel, and
  Shamir]{NIPS2013_SwitchingCosts}
Nicola Cesa-Bianchi, Ofer Dekel, and Ohad Shamir.
\newblock Online learning with switching costs and other adaptive adversaries.
\newblock In \emph{Advances in Neural Information Processing Systems(NIPS)},
  2013.

\bibitem[Combes et~al.(2015)Combes, Jiang, and
  Srikant]{Sigmetrics2015_BanditsWithBudgets}
Richard Combes, Chong Jiang, and Rayadurgam Srikant.
\newblock Bandits with budgets: Regret lower bounds and optimal algorithms.
\newblock In \emph{International Conference on Measurement and Modeling of
  Computer Systems (SIGMETRICS)}, 2015.

\bibitem[Dekel et~al.(2014)Dekel, Ding, Ding, Koren, and
  Peres]{STOC2014_DekelJian}
Ofer Dekel, Jian Ding, Jian Ding, Tomer Koren, and Yuval Peres.
\newblock Bandits with switching costs: $t2/3$ regret.
\newblock In \emph{ACM Symposium on Theory of computing (STOC)}, pages 459 --
  467, 2014.

\bibitem[Ding et~al.(2013)Ding, Zhang, and
  Liu]{AAAI2013_MABWithBdgetContraints}
Wenkui Ding, Tao Qin Xu-Dong Zhang, and Tie-Yan Liu.
\newblock Multi-armed bandit with budget constraint and variable costs.
\newblock In \emph{Proceedings of the Twenty-Seventh AAAI Conference on
  Artificial Intelligence}, 2013.

\bibitem[Gong and Shroff(2019)]{gong2019truthful}
Xiaowen Gong and Ness~B Shroff.
\newblock Truthful data quality elicitation for quality-aware data
  crowdsourcing.
\newblock \emph{IEEE Transactions on Control of Network Systems}, 7\penalty0
  (1):\penalty0 326--337, 2019.

\bibitem[Gopalakrishnan et~al.(2016)Gopalakrishnan, Doroudi, Ward, and
  Wierman]{gopalakrishnan2016routing}
Ragavendran Gopalakrishnan, Sherwin Doroudi, Amy~R Ward, and Adam Wierman.
\newblock Routing and staffing when servers are strategic.
\newblock \emph{Operations research}, 64\penalty0 (4):\penalty0 1033--1050,
  2016.

\bibitem[Gy{\"o}rgy et~al.(2007)Gy{\"o}rgy, Kocsis, Szab{\'o}, and
  Szepesv{\'a}ri]{grygor}
Andr{\'a}s Gy{\"o}rgy, Levente Kocsis, Ivett Szab{\'o}, and Csaba
  Szepesv{\'a}ri.
\newblock Continuous time associative bandit problems.
\newblock In \emph{IJCAI}, pages 830--835, 2007.

\bibitem[Hanawal et~al.(2015{\natexlab{a}})Hanawal, Leshem, and
  Saligrama]{ICAASP2015_CostEffectiveSpectral}
Manjesh~K. Hanawal, Amir Leshem, and Venkatesh Saligrama.
\newblock Cost effective algorithms for spectral bandits.
\newblock In \emph{IEEE International Conference on Acoustics, Speech and
  Signal Processing (ICASSP)}, pages 1323--1329, 2015{\natexlab{a}}.

\bibitem[Hanawal et~al.(2015{\natexlab{b}})Hanawal, Saligrama, Valko, and
  Munos]{ICML2015_CheapBandits}
Manjesh~K. Hanawal, Venkatesh Saligrama, Michal Valko, and Remi Munos.
\newblock Cheap bandits.
\newblock In \emph{International Conference on Machine Learning (ICML)},
  2015{\natexlab{b}}.

\bibitem[Immorlica et~al.(2019)Immorlica, Sankararaman, Schapire, and
  Slivkins]{FOCS2019_AdversarialBandits}
Nicole Immorlica, Karthik~Abinav Sankararaman, Robert Schapire, and Aleksandrs
  Slivkins.
\newblock Adversarial bandits with knapsacks.
\newblock In \emph{Annual Symposium on Foundations of Computer Science (FOCS)},
  2019.

\bibitem[Jun(2004)]{DeEconomist2004_Juan}
Tackseung Jun.
\newblock A survey on the bandit problem with switching costs.
\newblock In \emph{De Economist}, page 513–541, 2004.

\bibitem[Kalyanakrishnan et~al.(2012)Kalyanakrishnan, Tewari, Auer, and
  Stone]{ICML12_PACSubsetMAB}
Shivaram Kalyanakrishnan, Ambuj Tewari, Peter Auer, and Peter Stone.
\newblock Pac subset selection in stochastic multi-armed bandits.
\newblock In \emph{Proceedings of the 29th International Conference on
  International Conference on Machine Learning}, page 227?234, 2012.

\bibitem[Kay(1993)]{kay1993fundamentals}
Steven~M Kay.
\newblock \emph{Fundamentals of statistical signal processing}.
\newblock Prentice Hall PTR, 1993.

\bibitem[Kleinberg and Immorlica(2018)]{kleinberg2018recharging}
Robert Kleinberg and Nicole Immorlica.
\newblock Recharging bandits.
\newblock In \emph{2018 IEEE 59th Annual Symposium on Foundations of Computer
  Science (FOCS)}, pages 309--319. IEEE, 2018.

\bibitem[Lattimore and Szepesv\'ari(2019)]{Book2019_BanditAlgorithms}
Tor Lattimore and Csaba Szepesv\'ari.
\newblock Bandit algorithms.
\newblock 2019.

\bibitem[Lattimore and Szepesv{\'a}ri(2020)]{lattimore2020bandit}
Tor Lattimore and Csaba Szepesv{\'a}ri.
\newblock \emph{Bandit algorithms}.
\newblock Cambridge University Press, 2020.

\bibitem[Mannor and Tsitsiklis(2004)]{mannor2004sample}
Shie Mannor and John~N Tsitsiklis.
\newblock The sample complexity of exploration in the multi-armed bandit
  problem.
\newblock \emph{Journal of Machine Learning Research}, 5\penalty0
  (Jun):\penalty0 623--648, 2004.

\bibitem[Pike-Burke and Grunewalder(2019)]{pike2019recovering}
Ciara Pike-Burke and Steffen Grunewalder.
\newblock Recovering bandits.
\newblock In \emph{Advances in Neural Information Processing Systems}, pages
  14122--14131, 2019.

\bibitem[Seldin et~al.(2014)Seldin, Bartlett, Crammer, and
  Abbasi-Yadkori]{ICML14_seldin2014prediction}
Yevgeny Seldin, Peter~L Bartlett, Koby Crammer, and Yasin Abbasi-Yadkori.
\newblock Prediction with limited advice and multiarmed bandits with paid
  observations.
\newblock In \emph{ICML}, pages 280--287, 2014.

\bibitem[Trapeznikov and Saligrama(2013)]{AISTATS13_trapeznikov2013supervised}
Kirill Trapeznikov and Venkatesh Saligrama.
\newblock Supervised sequential classification under budget constraints.
\newblock In \emph{Artificial Intelligence and Statistics}, pages 581--589,
  2013.

\bibitem[Whittle(1988)]{whittle1988restless}
Peter Whittle.
\newblock Restless bandits: Activity allocation in a changing world.
\newblock \emph{Journal of applied probability}, pages 287--298, 1988.

\bibitem[Zhou et~al.(2018)Zhou, Gan, Yang, and Shen]{IJCAI2018_CasadingBandits}
Ruida Zhou, Chao Gan, Jing Yang, and Cong Shen.
\newblock Cost-aware cascading bandits.
\newblock In \emph{International Joint Conference on Artificial Intelligence
  (IJCAI)}, 2018.

\bibitem[Zolghadr et~al.(2013)Zolghadr, Bart{\'o}k, Greiner, Gy{\"o}rgy, and
  Szepesv{\'a}ri]{NIPS13_zolghadr2013online}
Navid Zolghadr, G{\'a}bor Bart{\'o}k, Russell Greiner, Andr{\'a}s Gy{\"o}rgy,
  and Csaba Szepesv{\'a}ri.
\newblock Online learning with costly features and labels.
\newblock In \emph{Advances in Neural Information Processing Systems}, pages
  1241--1249, 2013.

\end{thebibliography}

%\newpage
%\noindent\rule{\linewidth}{4pt} 
%\vspace{1.5mm}

%\centerline{\Large \bf Supplementary: Continuous Time Bandit with Sampling Cost}
%\hrulefill \\
\appendix
\clearpage
\vspace{10in}
 \section{Remarks on the system model}
  \begin{remark}\label{rem:horizon}
  Unlike the DMAB setting, where algorithms like UCB or Thompson sampling can work without the knowledge of $T$, in the current setting, the CTSAB algorithm we propose, crucially uses the information about $T$ to define phases and its decisions. Developing an algorithm without the knowledge  of $T$ for the  CTMAB appears challenging and is part of ongoing work.
\end{remark}
\section{Pseudo Code for Algorithm (\textsc{CTSAB})}
We use the notation $\b1_{x}=1$ if $x=1$ and $0$ otherwise. 
\begin{algorithm}[H]
		\caption{Continuous Time Single Arm Bandit (\textsc{CTSAB})}
		\label{alg:CTBandit}
		\begin{algorithmic}[1]
			\STATE {\bf Input } $0<\epsilon<1$, $\kappa > 1$, $T, \delta$ 
			\STATE\%Learning Period Starts
			\FOR {$i=1,2,3, 	 \ldots$} 
			\STATE Obtain $N_i = \kappa (\log T )T^{(2/3)i\epsilon}$ samples at uniform frequency in interval $[ T^{(i-1)\epsilon}-\b1_{\{i=1\}}, \ \ T^{i \epsilon} ]$
			\STATE $N^i = \sum_{j\le i} N_j, {\hat \mu}_i = \frac{1}{N^i}\sum_{j=1}^{N^i} X_i$,
				\IF{$\sqrt{\frac{\log(2/\delta)}{N^i}} < \frac{{\hat \mu}_i}{2}$}
					\STATE Set $i^\star=i$ and Break;
					\ENDIF
			\ENDFOR
			%\STATE \%Learning Period Ends
			\STATE \%Exploit Period Starts
			\STATE The time at which learning period ends  be $\tau$
			\IF{$\tau  \ge T$}
				\STATE Break;
					\ELSE
			\FOR{phases \ $k=i^\star+1, i^\star+2, 	 \ldots$}
				\STATE Obtain $N_k = {\hat \mu}_{k-1}T^{i^\star \epsilon}/2$  samples at equal intervals in
				$[\tau + (k-1)T^{i^\star \epsilon}, \tau + kT^{i^\star \epsilon}]$
					\STATE $N^k = \sum_{j\le k} N_j,{\hat \mu}_k = \frac{1}{N^k}\sum_{j=1}^{N^k} X_i$,
					\IF{$\tau + kT^{i^\star \epsilon} \ge T$}
				\STATE Break;
					\ENDIF
				\ENDFOR
				\ENDIF
				%\STATE \%Exploit Period Ends
		\end{algorithmic}
		% \vspace{-0.4cm}%
	\end{algorithm}
 \section{Preliminaries}
 % We begin with following preliminaries. %For analysis purposes, we will need the following Chernoff bound.
Let $X_i$'s be independent and identically Bernoulli distributed random variables with mean $\mu$, and ${\hat \mu}_N = \frac{1}{N} \sum_{i=1}^N X_i$.
\begin{lemma}\label{lemm:conc}
(Chernoff Bound) $\bbP\left(|{\hat \mu}_N- \mu| > \theta\right) \le \exp^{-2N \theta ^2}.$
%(Necessary Condition) If $\mu \le \frac{1}{4}$, then 
%$$P\left(|{\hat \mu}_N- \mu| > t\right) > \frac{1}{4}\exp^{-\frac{2n^2t^2}{\mu}}.$$
\end{lemma}
\begin{corollary}\label{cor:conc}Choosing $\theta= \frac{\log\left(\frac{1}{\delta}\right)}{\sqrt{N}}$, we get that $\text{err}_N = |{\hat \mu}_N- \mu| = O\left(\frac{\log\left(\frac{1}{\delta}\right)}{\sqrt{N}}\right)$ 
with probability at least $1-\delta$.
\end{corollary}

% \section{PRISM Algorithm \cite{prism}}
% 
%\begin{algorithm}
%	\caption{PRISM}
%	\label{alg:prism}
%	\begin{algorithmic}[1]
%		\STATE Input $\delta$. Let $A_1=\{ 1, \dots, K\}$ (All arms),
%		$n_\ell = \ell 2^\ell$, and $\epsilon_\ell = \sqrt{\frac{\log(1/\delta)}{2}}$
%		\STATE For each phase $\ell=1,2,\dots, $
%		\STATE 1. Let $i_\ell$ be the output of the Median Elimination Algorithm \cite{even2002pac} run on $A_\ell$ with accuracy $(\epsilon_\ell, \delta^\ell)$. 
%		\STATE 2. For each arm $i \in A_\ell$, sample arm $i$, $n_\ell$ times and let ${\hat \mu}_i(\ell)$ be its sample average 
%		\STATE 3. Let $A_{\ell+1}= \{ i \in A_\ell : {\hat \mu}_i(\ell) \ge {\hat \mu}_{i_\ell} -2 \epsilon_\ell\}$ 		\STATE Stop when $A_\ell$ contains a unique arm ${\hat i}$, output ${\hat i}$ as the best arm
%	\end{algorithmic}
%	% \vspace{-0.4cm}%
%\end{algorithm}
 \section{Proof of Theorem~ \ref{thm:ubsinglearm}}
 Throughout we need the definition of $p^\star$ \eqref{defn:p}.
 To prove Theorem \ref{thm:ubsinglearm}, we will need the following two Lemmas.
\begin{lemma}\label{lem:ublearningperiod}
For $\delta=\frac{1}{T^2}$, the learning period of algorithm \textsc{CTSAB} ends in at most $\left(\frac{p^\star}{\epsilon}+1\right)$ phases, i.e., $i^\star \le \left(\frac{p^\star}{\epsilon}+1\right)$, 
with probability at least 
$1-1/T^2$, where $\epsilon$ is such that $T^\epsilon = \sfc$ (constant).
\end{lemma}

Note that $\frac{p^\star}{\epsilon}$ need not be an integer, and to be precise, we should use $\lceil\frac{p^\star}{\epsilon}\rceil$. For ease of exposition, however, we ignore the ceiling.
Lemma \ref{lem:ublearningperiod} also shows that algorithm \textsc{CTSAB} does not $\textsf{fail}$ with probability at least $1-1/T^2$.  	%Suppose assume to the contrary. 
 \begin{proof}	
 	By the definition of phases, by the end of phase 
 	$\ell = \frac{p^\star}{\epsilon}+1$, the number of samples
 	obtained by the algorithm is $N^\ell \ge (\kappa \log (T)) T^{2/3 (p^\star+\epsilon)}$, where $p^\star$ has been defined in \eqref{defn:p}. 
 	%Therefore $\text{err}_{\ell} \le 
 	%\sqrt{\frac{\log(2/\delta)}{N_\ell}}$ for 
 	%by the end of phase $\ell$ with probability at 
 	%least $1-\delta$. 
 	Moreover, from the definition of $p^\star$ \eqref{defn:p}, we have 
 	$\frac{1}{\mu^4 T^{p^\star}} \le  
 	\left(\mu^2 T^{p^\star}\right)$.
 	This implies that 
 	\begin{equation}\label{eq:dummy1}
 	\sqrt{\frac{1}{T^{2p^\star/3}}} \le \mu.
 	\end{equation} Let $\delta = 1/T^2$. Then 
	\begin{align}\nn
 	\sqrt{\frac{\log(2/\delta)}{N^\ell}} &\le \sqrt{\frac{\log(2T^2)}{(\kappa \log (T)) T^{2/3 (p^\star+\epsilon)}}}\\ \label{eq:dummyx1x2}&  \le \mu\sqrt{\frac{\log(2T^2)}{ (\kappa \log (T)) T^{(2/3)\epsilon}}},
	\end{align} 
	where the final inequality follows from \eqref{eq:dummy1}. For any fixed $\epsilon >0$, such that $T^\epsilon = \sfc$ a constant, for a large enough constant $\kappa$,  we get  $ \mu\sqrt{\frac{\log(2T^2)}{(\kappa \log (T))T^{(2/3)\epsilon}}}\le \frac{\mu}{3}$. Using this fact in \eqref{eq:dummyx1x2}, we get 
 	\begin{equation}\label{eq:dummy2}\sqrt{\frac{\log(2/\delta)}{N^\ell}} \le \frac{\mu}{3}.
 	\end{equation}
 	
 	Recall from Corollary \ref{cor:conc} that $\bbP\left(|{\hat \mu}_\ell - \mu | = \text{err}_\ell \ge 
 	\sqrt{\frac{\log(2/\delta)}{N^\ell}}\right)\le \delta$. 
 	%Therefore, $\text{err}_\ell < \frac{\mu}{4}$ with 
 	%probability at least $1-\delta$. 
 	Thus, with probability at least $1-\delta$, we have 
	\begin{align}\nn {\hat \mu}_\ell \ge \mu - \text{err}_\ell &\ge \mu- \sqrt{\frac{\log(2/\delta)}{N^\ell}},\\ \nn
	& \stackrel{(a)}\ge  
 	3\sqrt{\frac{\log(2/\delta)}{N^\ell}} - \sqrt{\frac{\log(2/\delta)}{N^\ell}},\\ \nn
	& = 2 \sqrt{\frac{\log(2/\delta)}{N^\ell}}
	\end{align} at the end of phase $\ell= \frac{p^\star}{\epsilon}+1$, where $(a)$ follows from \eqref{eq:dummy2}. 
 	%Since $\sqrt{\frac{\log(2/\delta)}{N^\ell}}\le \text{err}_\ell < \frac{\mu}{4} \le \frac{{\hat \mu}_\ell}{3}$.
 	Thus, the learning period of algorithm \textsc{CTSAB} is completed by the $\ell^{th}$ phase, and the algorithm \textsc{CTSAB} never $\textsf{fails}$ with probability at least $1-\delta$, with $\delta=\frac{1}{T^2}$.
	\end{proof}

 	From the definition of algorithm \textsc{CTSAB}, the total number of samples $N^{i^\star}$ obtained by it in the learning period satisfies 
 	\begin{align}\label{eq:upplowboundsampCTSAB}
(\kappa \log (T))T^{(2/3) i^\star \epsilon } 
 &\le N^{i^\star}=  \sum_{i=1}^{i^\star} (\kappa \log (T)) T^{2/3 i \epsilon }\\
&\le (\kappa \log (T))T^{2/3 (i^\star+1) \epsilon }. \nonumber
\end{align}
	since the learning phase gets over in phase $i^\star$ and in each phase  $(\kappa \log (T))T^{(2/3) i^\star \epsilon }$ samples are obtained. Using this bound we get the following result.
 	\begin{lemma}\label{lem:payoff1}
  The total payoff of the  \textsc{CTSAB} algorithm in the learning period is  
 $$	P_L \ge \mu(\kappa \log (T)) T^{(2/3) i^\star \epsilon } - (\kappa \log (T))^2 T^{(i^\star+1)\epsilon/3}.$$
\end{lemma}
\begin{proof} 
%Let the time at which the learning period of algorithm \textsc{CTSAB} end be $t_{\text{learning}}$. 
The total payoff of the  \textsc{CTSAB} algorithm in the learning period by counting the payoff in each of the phase of the learning phases is 
  \begin{align}\nn
  P_L& = \mu N^{i^\star} - \sum_{i=1}^{i^\star} \frac{((\kappa \log (T)) T^{2/3 i \epsilon })^2}{T^{ i \epsilon} - T^{ (i-1) \epsilon}} \\\nn
	 	&\ge  \mu(\kappa \log (T)) T^{(2/3) i^\star \epsilon } -  \sum_{i=1}^{i^\star} \frac{((\kappa \log (T)) T^{2/3 i \epsilon })^2}{T^{ i \epsilon}/2}, \\\label{eq:payoff1}
 	& = \mu(\kappa \log (T)) T^{(2/3) i^\star \epsilon } - 2(\kappa \log (T))^2 T^{(i^\star+1)\epsilon/3},
 	\end{align} 
	where the first inequality follows since $T^{ i \epsilon} - T^{ (i-1) \epsilon} \ge T^{ i \epsilon}/2$, while the number of samples obtained by algorithm \textsc{CTSAB} in the learning period is as given by \eqref{eq:upplowboundsampCTSAB}.
\end{proof}

 	%Note that the number of samples $N_1= i^\star T^{(2/3) i^\star \epsilon }= O(1/\mu^2)$ as shown in Lemma \ref{lem:ublearningperiod}.

 Next simple lemma helps to show  that once the learning period is complete in a particular phase, in subsequent 
phases the payoff obtained by  \textsc{CTSAB} algorithm is positive. 
\begin{lemma}\label{lem:direct}
Let the number of samples obtained be $N$, and 
$|{\hat \mu}_N- \mu | = \text{err}_N$.  If 
$\text{err}_N<  \frac{{\hat \mu}_N}{2}$, then 
$\text{err}_N< \mu$.
\end{lemma}
\begin{proof}
By definition,  $\mu = {\hat \mu}_N - \text{err}_N$. 
Therefore, if $\text{err}_N< \frac{{\hat \mu}_N}{2}$, it implies that  ${\hat \mu}_N - \text{err}_N 
> \text{err}_N $ which  is sufficient for 
$\text{err}_N< \mu$.
\end{proof}
 
 Now we complete the Proof of Theorem~ \ref{thm:ubsinglearm}.
 	From Corollary \ref{cor:conc}, at the end of phase $i$, $\text{err}_i \le \sqrt{\frac{\log(2/\delta)}{N^{i}}}$ 
 	with probability at least $1-\delta$, where $N^i$ is total number of samples obtained until the end of phase $i$. Thus, 
 	the condition that ends the  training period in phase $i^\star$, $\sqrt{\frac{\log(2/\delta)}{N^{i^\star}}} < \frac{{\hat \mu}_{i^\star}}{2}$ ensures that $\text{err}_{i^\star}< \frac{{\hat \mu}_{i^\star}}{2}$ with probability at least $1-\delta$, which implies that 
 	$\text{err}_{i^\star} < \mu$ (Lemma \ref{lem:direct}). For phase $i> i^\star$, obtaining ${\hat \mu}_iT^{i^\star \epsilon}/2$ samples in each phase, the payoff obtained by algorithm \textsc{CTSAB} in phase $i$ is 
 	\begin{align}\nonumber
 	P_i & = \mu\frac{{\hat \mu}_iT^{i^\star \epsilon}}{2}  - 
 	\frac{(\frac{{\hat \mu}_iT^{i^\star \epsilon}}{2})^2}{T^{i^\star \epsilon}},\\ \nn
 	& = \mu\frac{ (\mu \pm \text{err}_i)T^{i^\star \epsilon}}{2}  - 
 	\frac{(\frac{(\mu \pm \text{err}_i)T^{i^\star \epsilon}}{2})^2}{T^{i^\star \epsilon}},\\ \label{eq:payoffwitherr}
 	&\ge \frac{\mu^2T^{i^\star \epsilon}}{4} - \text{err}_i^2\frac{T^{i^\star \epsilon}}{4}. \end{align}
 	Since at the end of learning period phase $i^\star$, $\text{err}_{i^\star} < \mu$, hence, 
 	we have that $P_i> 0$ for $i > i^\star$.

 	Let the {\it bad} event in phase $j$ be defined as $B_{j} = \{\text{err}_{j} > \frac{c}{\sqrt{N^{j}}}\}$, where $N^{j}$ is the sum of the number of samples obtained until the end of 
 	phase $j$. For further analysis of the payoff of algorithm \textsc{CTSAB} in the exploit period, we want to bound the probability that a bad event happens during any phase (both in learning and exploit period). 
	\begin{lemma}\label{lem:bad}
  The probability that a {\it bad} event happens in any phase of learning or exploit period of algorithm \textsc{CTSAB} is $\le 1-\frac{T}{T^{ \epsilon}}\delta$ for $c= \sqrt{\log(2/\delta)/2}$.
\end{lemma}
	\begin{proof}
  From Corollary \ref{cor:conc}, we 
 	know that $\bbP(B_j) \le \delta$, if $c= \sqrt{\log(2/\delta)/2}$.
 	Thus, the probability that in any phase (both in learning period and beyond), a bad event happens 
 	$\bbP(\cup_j B_j) \le \frac{T}{T^{\epsilon}} \bbP(B_j) \le 
 	\frac{T}{T^{ \epsilon}}\delta$, since there are at most $\frac{T}{T^{ \epsilon}}$ phases in all (counting both the learning period and the exploit period). \end{proof}
	
	From here on, we will assume that during no phase a bad event happens, and account for its probability $1-\frac{T}{T^{ \epsilon}}\delta$ appropriately.

 	Recall that with the exploit period of algorithm  \textsc{CTSAB}, assuming ${\hat \mu}_i$ to the true value of $\mu$, the number of samples to be obtained in phase $i+1, i\ge i^\star$, is given by $\frac{{\hat \mu}_iT^{i^\star \epsilon}}{2}$. 
 	Thus, the total number of samples obtained by the end of phase 
 	$ i> i^\star$ (for all $ i> i^\star$) is given by
 	\begin{align}\nn
 	N^i & \ge (\kappa\log T)T^{(2/3) i^\star \epsilon } +\sum_{j=i^\star+1}^i {\hat \mu}_j\frac{T^{i^\star \epsilon}}{2}, \\ \nn
 	& =  (\kappa\log T)T^{(2/3) i^\star \epsilon } +\sum_{j=i^\star+1}^i  (\mu \pm \text{err}_j)\frac{T^{i^\star \epsilon}}{2},\\\nn
 	& \stackrel{(a)}\ge  (\kappa\log T)T^{(2/3) i^\star \epsilon } +\sum_{j=i^\star+1}^i  \left(\mu \pm \frac{c}{\sqrt{N_{j-1}}}\right)\frac{T^{i^\star \epsilon}}{2}, \\ \label{eq:dum999}
 	%& \mbox{\textcolor{red}{MH Says: $err_j \leq c/\sqrt{N_{j-1}}$ holds with some probablity,  this additional probability should be considered?}}\\
 	&\ge \mu T^{i^\star \epsilon}\frac{i-i^\star}{2} -o(T^{i^\star \epsilon}),
 	%& \mbox{\textcolor{red}{MH Says: why only negative term is retained?}}
 	\end{align}
	where $(a)$ follows from Lemma \ref{lem:bad} $\forall j \le i$.
 	%\begin{remark}
 	%We know that the estimator for $\mu$, ${\hat \mu}$ is unbiased, 
 	%thus $\mathbb{E}\{N_i\} \ge \mu T^{i^\star \epsilon}\frac{i-1}{4\lambda}.$
 	%Question : What does this translate in terms of 
 	%$P\left(\text{err}_i < \right)$ with probability at least 
 	%$1-\delta$.
 	%\end{remark}
 	Therefore using \eqref{eq:dum999}, from Lemma \ref{lem:bad}, we get $\text{err}_i \le \sqrt{\frac{2 \ln(2/\delta)}
 		{2\mu (i-i^\star) T^{i^\star \epsilon}}}$ with probability at least $1-\frac{T}{T^{ \epsilon}}\delta$ for all phases 
 	$i\ge i^\star+1$.
 	Thus, with probability at least $1-\frac{T}{T^{ \epsilon}}\delta$, for phase $ i\ge i^\star+1$, following \eqref{eq:payoffwitherr}, the payoff obtained in phase $i$
 	\begin{align}\nn
P_i 
 	&\ge \frac{\mu^2T^{i^\star \epsilon}}{4} - \frac{2 \ln(2/\delta)}
 	{2\mu (i-i^\star) T^{i^\star \epsilon}}\frac{T^{i^\star \epsilon}}{4}, \\\label{eq:dummy900}
 	&\ge \frac{\mu^2T^{i^\star \epsilon}}{4}
 	-\frac{ \ln(2/\delta)}{4\mu (i-i^\star) }.
\end{align}
 	Algorithm \textsc{CTSAB} fails with probability at most $\delta$, and the probability that any {\it bad} event 
	happens is at most $\frac{T}{T^{ \epsilon}}\delta$. Thus, with probability $1-\delta - \frac{T}{T^{ \epsilon}}\delta$ neither the Algorithm \textsc{CTSAB} fails nor any of the {\it bad} events happen.
 	Thus, combining \eqref{eq:payoff1} and \eqref{eq:dummy900} the total payoff for the  \textsc{CTSAB} algorithm  is 
 	\begin{align}\nn
 	P &= P_L+\sum_{i=i^\star+1}^{\frac{T-T^{i^\star \epsilon}}{T^{i^\star \epsilon}}} P_i, \\\nn
 	& \ge \mu(\kappa \log (T)) T^{(2/3) i^\star \epsilon } - 2(\kappa \log (T))^2 T^{(i^\star+1)\epsilon/3} + \sum_{i=i^\star+1}^{\frac{T-T^{i^\star \epsilon}}{T^{i^\star \epsilon}}}
 	\frac{\mu^2T^{i^\star \epsilon}}{4}
 	-\sum_{i=i^\star+1}^{\frac{T-T^{i^\star \epsilon}}{T^{i^\star \epsilon}}}\frac
 	{\ln(2/\delta)}
 	{4\mu (i-i^\star) }, \\ \label{eq:finalpayoffsinglearm}
 	& \ge  \mu(\kappa \log (T)) T^{(2/3) i^\star \epsilon } - 2(\kappa \log (T))^2 T^{(i^\star+1)\epsilon/3} + 
 	\frac{\mu^2(T- T^{(i^\star+1) \epsilon})}{4}
 	-\frac
 	{\ln(2/\delta)}
 	{4\mu  } \ln \left(\frac{T}{T^{i^\star \epsilon}}\right),
 	\end{align}
 	with probability at least $1-\frac{T}{T^{ \epsilon}}\delta - \delta$, where in the final inequality, the third term follows since the total length of the learning period is at most 
 	$T^{(i^\star+1) \epsilon}$, while in the final term we have used a simple upper bound $\sum_{i=i^\star+1}^{\frac{T-T^{i^\star \epsilon}}{T^{i^\star \epsilon}}}\frac
 	{1}
 	{ (i-i^\star) } \le \ln \left(\frac{T}{T^{i^\star \epsilon}}\right)$.
 	Recall that the payoff of the oracle policy is $\frac{\mu^2T}{4}$. Thus the regret of  \textsc{CTSAB} algorithm is at most 
 	\begin{align}\nn
 	R&\le   2(\kappa \log (T))^2 T^{(i^\star+1)\epsilon/3} 
 	+\frac{\mu^2  T^{(i^\star+1) \epsilon}}{4}+ 	\frac
 	{\ln(2/\delta)}
 	{4\mu  } \ln \left(\frac{T}{T^{i^\star \epsilon}}\right), \\ \nn
 	&\stackrel{(a)}= O( (\log (T))^2T^{p^\star/3} T^{4\epsilon/3}) + O((\mu^2)  T^{p^\star+\epsilon}) 
 	O(\mu^2  T^{p^\star+\epsilon/3}\ln(2/\delta)\ln (T)),\\ \nn
 	&\stackrel{(b)}= O\left((\log (T))^2\frac{1}{\mu}T^{4\epsilon/3}\right) +  O(\mu^2 T^{p^\star + \epsilon'} \ln(2/\delta)\ln (T)), \\ \label{eq:dummyxx2}
	&\stackrel{(c)}=  O((\log (T))^2 \mu^2 T^{p^\star}T^{4\epsilon/3}) + O(\mu^2 T^{p^\star + \epsilon'} \ln(2/\delta)\ln (T)),
	%&=O(\mu^2 T^{p^\star + \epsilon'} \ln(2/\delta)\ln (T)),
 	\end{align}
 	with probability $1-\frac{T}{T^{ \epsilon}}\delta - \delta$, where in $(a)$ the first term follows since 
 	$i^\star \epsilon \le p^\star+\epsilon$ (Lemma \ref{lem:ublearningperiod}), and the third term using the definition of $p^\star$ \eqref{defn:p}, while to get $(b)$ we use \eqref{defn:p} and $\epsilon' = 4/3\epsilon$,  $(c)$ follows since $\left(\frac{1}{\mu}\right) = O\left(\mu^2T^{p^\star}\right)$ from \eqref{defn:p}. 
	
	%{\bf [XXX: This is a correction compared to submitted version]}
	
\begin{remark}\label{rem:univregretbound}	Given that the arm reward distribution is Bernoulli, $\text{err}_i$ is bounded at the end of each phase $i$ with the \textsc{CTSAB} algorithm. Thus, following \eqref{eq:imp}, the regret of \textsc{CTSAB} algorithm is at most $O(T)$.
	\end{remark}
	
	With $\delta = 1/T^2$, and since the 
 	maximum regret can be $O(T)$ (Remark \ref{rem:univregretbound}), from \eqref{eq:dummyxx2} we get that the expected regret of the  \textsc{CTSAB} algorithm
 	is 
 	\begin{align*}
 	\cR={\mathbb E}\{R\}& =  
 	O(\mu^2 T^{p^\star+ \epsilon'} \ln(2/\delta) \ln (T))(1-\frac{T}{T^{ \epsilon}}\delta -\delta) \\
 	& + O(T)  \left(\frac{T}{T^{ \epsilon}}\delta +\delta\right), \\
 	& =O(\mu^2 T^{p^\star+ \epsilon'} \ln(2 T^2)\ln (T))(1-\frac{T}{T^{ \epsilon}}\frac{1}{T^2} - \frac{1}{T^2}) \\
 	&+ 
 	O(T)\left( \frac{1}{T^{1+\epsilon}} + \frac{1}{T^2}\right), \\
 	&= O(\mu^2 T^{p^\star+ \epsilon'}\ln(2 T^2)\ln(2 T))(1-\frac{1}{T^{1+\epsilon}}-\frac{1}{T^2}) \\
 	&+ O(T)\left( \frac{1}{T^{1+\epsilon}} + \frac{1}{T^2}\right). 
 	\end{align*}
 	This completes the proof of Theorem \ref{thm:ubsinglearm}.
 	%\textcolor{red}{(MH Says: points to discuss in details: 1) Even knowing $p$ doesnot reveal value of $\mu$  2) The role of magnitude of $\mu$. Why we need to restrict $\mu$ to be less than $1$. 3) how this problem differs from the conventional MAB and analysis. )}
	 \section{Preliminaries to prove Theorem \ref{thm:lbsinglearm}}
 
To lower bound the regret of any algorithm for the single arm CTMAB,  we will need the following preliminaries.

%To derive a lower bound on the regret of any algorithm with multiple arms in Theorem \ref{thm:lbmultiplearms}, we need the following preliminaries. 

{\bf Prediction Problem :} Consider two Bernoulli distributions $\cP$ and $\cQ$ with means $x$ and $x+\beta$, respectively, where $x>0, \beta >0$. 
A coin is tossed repeatedly with probability of heads distributed according to either $\cP$ or $\cQ$, where repeated tosses are independent. 
From the observed samples, the problem is to predict the correct 
distribution $\cP$ or $\cQ$ such that the success probability of the prediction is at least $1-\gamma$.

\begin{lemma}\label{lem:prediction} 
  Let $n$ be fixed. Consider any algorithm $\cA_n$ that obtains $n$ samples and solves the prediction problem with success probability  at least $1-\gamma$. Then $n >\left( \frac{\ln 2}{ 8\beta^2} \log \left(\frac{1}{4\gamma}\right)\right)$ for $\beta<1/4$.
\end{lemma}
\begin{proof}
Let the product distribution over $n$ samples derived from $\cP$ and $\cQ$ be $\cP^n$ and $\cQ^n$.
  Let $E$ be the event that the algorithm $\cA_n$ outputs $\cP$ as the correct distribution, and $E^c$ be its complement.
  Then from Theorem 14.2 \cite{lattimore2020bandit}, we have that 
  \begin{equation}\label{eq:pred1}
\bbP_\cP(E^c) + \bbP_\cQ(E) \ge \frac{1}{2} \exp\left(-D(\cP^n || \cQ^n)\right),
\end{equation}
where $D(\cP^n || \cQ^n)$ is the Kullback-Liebler distance between $\cP^n$ and $\cQ^n$. Since the probability of success for $\cA_n$ is $\ge 1-\gamma$, we have that  both $\bbP_\cP(E^c) < \gamma$ and $\bbP_\cQ(E) < \gamma$. Thus, from \eqref{eq:pred1}, we get 
\begin{equation}\label{eq:pred2}
2\gamma >  \frac{1}{2} \exp\left(-D(\cP^n || \cQ^n)\right).
\end{equation}
Moreover, we have that $D(\cP^n || \cQ^n) = n D(\cP || \cQ)$, and  $D(\cP || \cQ) \le \frac{8\beta^2}{\ln 2}$ for $\beta\le 1/4$ for $\cP$ and $\cQ$ being Bernoulli with means $x$ and $x+\beta$.

Therefore, we get that 
 \begin{equation}\label{eq:pred3}
4\gamma > \exp\left(-\frac{n8\beta^2}{\ln 2}\right),
\end{equation}
which implies the result.
\end{proof}
In addition to Lemma \ref{lem:prediction}, we need the following Lemma that is specific to the considered problem. 
Let  an algorithm $\cA$ obtain $N_t$ samples in time $[0,t]$. Recall that   
${\hat \mu}_t = \frac{1}{N_t}\sum_{i=1}^{N_t} X_i$ is the empirical average of the reward obtained by $\cA$ using the $N_{t}$ 
samples, where $\bbE\{X_i\} = \mu, \ \forall \ i$
and $\text{err}_t = |{\hat \mu}_t - \mu|$ is the error in estimating $\mu$ by $\cA$ at time $t$. 
%Let $\bbP_f(t) = \bbP(\text{err}_t > \mu)$ be the probability that the error in estimating $\mu$ with algorithm $\cA$ at time $t$ is at least $\mu$. 
%
\begin{lemma}\label{lem:tptime}
	For an online algorithm $\cA$ (satisfying Assumption \ref{rem:unbiased}) the maximum payoff possible in interval $[0,t]$ is at most $\frac{\mu^2t}{4} - \text{err}_t^2\frac{t}{4}$. \end{lemma}
\begin{proof}
Recall that we are considering algorithms that only use  unbiased estimates of $\mu$ to make decisions as described in Remark \ref{rem:unbiased}.
For any algorithm $\cA$, let the (sample mean) estimate of $\mu$ at time $t$ be ${\hat \mu}_t$ where  
	$\text{err}_t = |{\hat \mu}_t - \mu|$. We want to upper bound the expected payoff of $\cA$ in $[0,t]$. Towards that end, we bound the expected payoff if $\cA$ knew ${\hat \mu}_t$ at time $0$ itself, which can only improve the lower bound.
	
%	let  ${\hat \mu}_t$ be available  at time $0$ itself, which can only improve the expected payoff of $\cA$ in $[0,t]$. Next, we argue that even if  ${\hat \mu}_t$ is available  at time $0$ itself for $\cA$, its expected payoff in $[0,t]$ when $\text{err}_t = |{\hat \mu}_t - \mu| \ge \mu$ is at most $0$.

Note that with Bernoulli distribution, the empirical estimate ${\hat \mu}_t$ is a sufficient and complete statistic \cite{kay1993fundamentals} for $\mu$, and the minimum variance unbiased estimator (MVUE) for $\mu$. 

Knowing ${\hat \mu}_t$ at time $0$, let $N({\hat \mu}_t)$ be the number of samples obtained  (at equal intervals since it minimizes the sampling cost) by algorithm $\cA$ in time $[0,t]$. Then maximizing the 
 expected payoff of $\cA$ in interval $[0,t]$ is equivalent to minimizing the expected regret \eqref{defn:regret} of $\cA$ in $[0,t]$, given by 
%From Proposition \ref{prop:oracle}, we know that the optimal choice of the number of samples to 
 %obtain in any interval of duration $t$ is $\mu t/2\lambda$. Thus, knowing ${\hat \mu}_t$, the number of samples to obtain that maximize the expected payoff  in any interval of size $t$ is ${\hat \mu} t/2\lambda$.
 
 	\begin{align}\nn
  \cR([0,t]) & =\min_{N({\hat \mu}_t)} \bbE\left\{\frac{\mu^2 t}{4} -  \mu N({\hat \mu}_t) + \frac{1}{t} N({\hat \mu}_t)^2\right\},\\ \nn
  & = \min_{N({\hat \mu}_t)}\bbE\left\{\left(\sqrt{\frac{t}{4}}\mu - \sqrt{\frac{1}{t}}N({\hat \mu}_t)\right)^2\right\},   \\ \label{eq:dummyxx1}
  & =\min_{N({\hat \mu}_t)}\bbE\left\{\frac{t}{4} \left(\mu - \frac{2}{t}N({\hat \mu}_t)\right)^2\right\}.
\end{align}
Since ${\hat \mu}_t$ is MVUE for $\mu$, the number of samples $N^\star({\hat \mu}_t)$ that $\cA$  obtains to minimize regret \eqref{eq:dummyxx1} (and maximize the expected payoff) knowing ${\hat \mu}_t$ at time $0$ in $[0,t]$   is ${\hat \mu}_t t/2$.

%where the final inequality is true since $\hat{\mu}_t\ge 2\mu.$
 
%	For algorithm $\cA$, let the estimate of $\mu$ at time $t$ be ${\hat \mu}_t$ such that 
%	$\text{err}_t = |{\hat \mu}_t - \mu| \ge \mu$. Making ${\hat \mu}_t$ available to $\cA$ at time $0$ itself, can only improve the performance of $\cA$. We assume that here after. Thus, for time period  $[0,t]$, the optimal 
%	number of samples $N_t$ to obtain is ${\hat \mu}_t t/2\lambda$. 
%	

Therefore, the upper bound on the payoff of $\cA$ in time $[0,t]$ is 
		\begin{align}\nn
  P([0,t]) & \le  \mu N^\star({\hat \mu}_t) -  (N^\star({\hat \mu}_t))^2/t,\\ \label{eq:imp}
   & = \frac{\mu^2t}{4} - \text{err}_t^2\frac{t}{4},
\end{align} where  $\text{err}_t = |{\hat \mu}_t-\mu|.$ 
%Since the expected payoff of $B$ is at least as good as $\cA$, the result follows. 

\end{proof}

%We will be using the following result that relates the number of samples needed to bound the probability of identifying the arm with the largest mean.  
%Lemma \ref{lem:lbpureexplore} gives a lower bound on the samples to be obtained by an algorithm to be $(\beta,\gamma)$-correct. Using Lemma \ref{lem:lbpureexplore} for $K=2$, next, we derive a lower bound on the number of samples to be obtained to solve the prediction problem. 
%%We will use Lemma \ref{lem:lbpureexplore} with $\beta=\Delta/2$ so that the set of $\beta$-optimal arms only contains the best arm, to derive the following result.
%
%
%\begin{lemma}\label{lem:prediction} For $\beta \le 1/8$, for the prediction problem, the minimum number of coin tosses satisfies $T_{\min} \ge  c_1 \frac{2}{\beta^2} \log \left(\frac{c_2}{\gamma}\right)$. 
%\end{lemma}
%%The proof uses Pinsker's inequality and can be found in Section 1.2 \cite{madhur}.
\section{Basic idea for proving Theorem \ref{thm:lbsinglearm}}\label{app:thm:lbsinglearm:idea}
Using Lemma \ref{lem:prediction}, we next derive a lower bound on the regret of any algorithm for the single arm CTMAB problem.
The basic idea used to derive the lower bound is that 
if suppose the regret of any algorithm $\cA$ is $\mu^2 T^{p}$, then we show that with $\cA$ it must be that at time $T^{p+\alpha}$ for some $\alpha>0$,  the probability that  $\text{err}_{T^{p+\alpha}} > \mu/T^{\alpha/4}$ (error in estimating $\mu$ at time $T^{p+\alpha}$ is greater than $\mu/T^{\alpha/4}$) is at most $1/T^{\alpha/2}$. This condition is necessary, since otherwise Lemma \ref{lem:tptime} implies that the regret of $\cA$ is $>\mu^2 T^{p}$ contradicting the regret bound of $\mu^2 T^{p}$. 
The necessary condition $\bbP(\text{err}_{T^{p+\alpha}} > \mu/T^{\alpha/4}) <  1/T^{\alpha/2}$ implies a lower bound on the number of samples to be obtained by $\cA$ in interval $[0,T^{p+\alpha}]$ from Lemma \ref{lem:prediction}.  Accounting for the sampling cost resulting out of this lower bound on the number of samples gives us the required lower bound on the regret. 

\section{Proof of Theorem \ref{thm:lbsinglearm}}
\begin{proof} Consider any online algorithm $\cA$ for the single arm CTMAB for which Assumption \ref{rem:unbiased} holds. 
Let the regret of $\cA$  be 
$g(\mu, T)$ which can be expressed as $g(\mu, T) = \mu^2 T^p$ for some $p<1$, since recall that the oracle payoff in time interval $[0,T]$ is $\mu^2 T/ 4$.

%Consider any algorithm $\cA$ and let its regret be $\mu^2 T^p$ for some $p < 1$ (Remark \ref{rem:trivialubregret}). Note that we are not assuming that $p$ is a constant,  and $p$ can depend on any parameter of the problem and hence can simulate any regret dependence on $\mu, T$ etc. We are assuming this structure for the regret just for the ease of analysis.

We divide the total time horizon $T$ into two {\bf intervals}, first $[0, T^{p+\alpha}]$, and second  $[T^{p+\alpha}, T]$. 

Let $\cA$ know that the true mean is either $\mu$ or $\mu+ \frac{2\mu}{T^{\alpha/4}}$ for some $\alpha>0$ that we will 
specify later. This assumption can only reduce the regret of any algorithm. Let the true mean be $\mu_{\text{true}} \in \{\mu,\mu+ \frac{2\mu}{T^{\alpha/4}}\}$. 

Let the number of samples obtained by $\cA$ till time $T^{p+\alpha}$ be $N_\cA(T^{p+\alpha})$ and $N_{\min} = \frac{c_1T^{\alpha/2}}{\mu^2} \log (c_2 T^{\alpha/2})$. Assume that $N_\cA(T^{p+\alpha})< N_{\min}$.
From Lemma \ref{lem:prediction}, if the number of samples $N_\cA(T^{p+\alpha})$ is less than $N_{\min}$, then $\bbP(\text{err}_{N_\cA(T^{p+\alpha})} = |{\hat \mu}_{N_{T^{p+\alpha}}} - \mu_{\text{true}}|  > \frac{\mu}{T^{\alpha/4}})> \frac{1}{T^{\alpha/2}}$. Note that even though $N_\cA(T^{p+\alpha})$ is a random variable (depending on the realizations seen by $\cA$ till time $T^{p+\alpha}$), we are applying Lemma \ref{lem:prediction} since it is true for any 
algorithm that obtains a given number of samples, in particular $N_\cA(T^{p+\alpha})$.

%With error probability  $> \frac{1}{T^{\alpha/2}}$ the regret is more than $\mu^2 T^p$ and hence the number of samples $\cA$ has to obtain in time $[0, T^{p+\alpha}]$ is at least $N_{\min}$. Rest of the proof follows.

Without loss of generality, let $\mu_{\text{true}} = \mu$. Moreover, define $\bbP_b(t) = \bbP(\text{err}_{N_{T^{p+\alpha}}} > \frac{\mu}{T^{\alpha/4}})$, and $\bbP_g(t)= 1-\bbP_b(t)$.

%For $\cA$, let $\text{err}_{N_{T^{p+\alpha}}} = |{\hat \mu}_{N_{T^{p+\alpha}}} - \mu|$ be the error in estimating $\mu$ at time $T^{p+\alpha}$ (end of the first interval)  when $\cA$ obtains $N_{T^{p+\alpha}}$ samples in time $[0,T^{p+\alpha}]$. 
%Since mean can only take two values $\mu$ and $2\mu$, $\text{err}_{T^{p+\alpha}} \in \{0, \mu\}$.
%Let the bad event probability be $\bbP_b(t) = \bbP(\text{err}_{t} \ge \frac{\mu}{T^{\alpha/4}})$ be the probability that the error in estimating $\mu$ with algorithm $\cA$ at time $t$ is at least $\frac{\mu}{T^{\alpha/4}}$. Let .

Then the  payoff of algorithm $\cA$ over the two intervals $[0, T^{p+\alpha}]$, and  $[T^{p+\alpha}, T]$ can be written as  
\begin{align}\nn P_\cA &= P_\cA([0,T^{p+\alpha})] + P_\cA([T^{p+\alpha}, T]), \\ \nn
&\stackrel{(a)} \le \bbP_g(T^{p+\alpha})\frac{\mu^2 (T^{p+\alpha})}{4} +  \bbP_b(T^{p+\alpha}) \cdot\left(\mu^2 \frac{T^{p+\alpha}}{4}-  
\frac{\mu^2}{T^{p+\alpha/2}} \frac{T^{p+\alpha}}{4}\right) + \frac{\mu^2 (T - T^{p+\alpha})}{4},  \\
\label{eq:singlearmpayoffdummy1}
  & \stackrel{(b)} =    -  \mu^2 (T^{p+\alpha/2})\frac{\bbP_b(T^{p+\alpha})}{4} +\mu^2 \frac{T}{4}.
\end{align}
where the first two terms of $(a)$ follow as detailed next, case i) when $\text{err}_{N_\cA(T^{p+\alpha})} \ge \frac{\mu}{T^{\alpha/4}}$ using \eqref{eq:imp}, while in case ii) when $\text{err}_{N_\cA(T^{p+\alpha})} < \frac{\mu}{T^{\alpha/4}}$, making $\text{err}_{N_\cA(T^{p+\alpha})}=0$ and getting the oracle payoff $\frac{\mu^2 (T^{p+\alpha})}{4}$, while for the third term we let  $P_\cA([T^{p+\alpha}, T])$ to be the oracle's payoff for interval $[T^{p+\alpha}, T]$.
As discussed earlier, given that $N_\cA(T^{p+\alpha}) < N_{\min}$, $\bbP_b(T^{p+\alpha})>\frac{1}{T^{\alpha/2}}$. Using this fact, $(b)$ implies that the regret ($\mu^2 T/4 - P_\cA$) of $\cA$ is larger than $\mu^2 T^p$ giving us a contradiction. Therefore, for $\cA$ to have regret $\mu^2T^{p}$, it is necessary that $N_\cA(T^{p+\alpha}) \ge N_{\min}$.

As a function of $N_\cA(T^{p+\alpha})$, we can write the payoff of $\cA$ as 
\begin{align}\nn P_\cA &= P_\cA([0,T^{p+\alpha})] + P_\cA([T^{p+\alpha}, T]), \\ \label{eq:singlearmpayoffdummy1}
&\stackrel{(a)} \le \mu N_\cA(T^{p+\alpha}) -   \frac{ N^2_\cA(T^{p+\alpha})}{T^{p+\alpha}}   +  \frac{\mu^2 (T-T^{p+\alpha})}{4},
\end{align}
where in $(a)$ for the first interval the bound follows since the sampling cost is smallest if $ N_\cA(T^{p+\alpha})$ samples are obtained at uniform intervals in $[0, T^{p+\alpha}]$, while for the second term $P_\cA([T^{p+\alpha}, T])$ assuming that the algorithm $\cA$
knows the true value of $\mu$ at time $T^{p+\alpha}$, and obtains the payoff  equal to that of the oracle policy for interval $[T^{p+\alpha}, T]$.
%Recall that the payoff of the first interval 
%\begin{equation}\label{eq:lbsingledum1}
%P([0,T^{p+\alpha}])  \le \frac{\mu^2 T^{p+\alpha}}{4\lambda}
%\end{equation} from Proposition \ref{prop:oracle}. 

	\begin{remark}\label{rem:unimodsingle}
 From \eqref{eq:singlearmpayoffdummy1}, the payoff in the first interval $\mu x -   \frac{ x^2}{T^{p+\alpha}}$ (where $x$ is the number of samples obtained in time $[0, T^{p+\alpha}])$ is a concave and unimodal function of $x$ with optimal $x^\star =  \frac{\mu T^{p+\alpha}}{2}$.\end{remark}
 Thus, we consider two cases : i) $x^\star \ge N_{\min}$ and ii) $x^\star  <N_{\min}$. 

Case i) When  $x^\star \ge N_{\min} $, using expressions for $x^\star$ and  $N_{\min} $, we get 
that $$\frac{\mu T^{p+\alpha}}{2}  \ge c_1T^{\alpha/2} \log (c_2 T^{\alpha/2})/\mu^2,$$ implying 
\begin{equation}\label{eq:pbound1}
T^{p+\alpha} \ge  \frac{2c_1 T^{\alpha/2}}{\mu^3} \log(c_2T^{\alpha/2})
  \quad \forall \ \alpha>0.
\end{equation}

Choosing $T^{\alpha/2}= c > 1$ (a constant), i.e. $\alpha = \frac{2\log c}{\log(T)}$, gives 
$$T^p \ge \frac{2c_1}{c\mu^3}\log(c_2) +  \frac{2c_1}{c\mu^3}2\log(c)\frac{\log(T)}{\log(T)}.$$ Thus, 
\begin{equation}\label{eq:pbound2}
p \ge \min \left\{k\ge 0: T^k \ge \frac{\sfc_1}{\mu^3} \right\},
\end{equation}
where $\sfc_1$ is a constant.
For case ii) we proceed as follows. We have already argued that $N_\cA(T^{p+\alpha})\ge N_{\min}$. 
Thus, following Remark \ref{rem:unimodsingle}, with $x^\star  < N_{\min}$, 
it is clear that the RHS of \eqref{eq:singlearmpayoffdummy1}  is a decreasing function of $N_{T^{p+\alpha}}\ge N_{\min}$ for fixed $\mu$ and $T$. Thus, choosing $N_\cA(T^{p+\alpha}) = N_{\min}$ (the minimum possible), from \eqref{eq:singlearmpayoffdummy1}, we obtain the largest expected payoff of algorithm $\cA$ for the first interval, and the overall expected payoff over the two intervals is 
\begin{align}\nn P_\cA &\le \mu N_{\min} -   \frac{1}{T^{p+\alpha}} N_{\min}^2   +  \frac{\mu^2 (T-T^{p+\alpha})}{4}.
%& = \frac{c_3}{\mu} -   \frac{1}{T^{p+\alpha}} N_{\min}^2   +  \frac{\mu^2 (T-T^{p+\alpha})}{4},
\end{align}
Thus, for the regret of algorithm $\cA$ to be $\mu^2 T^p$, we need $$\frac{ N_{\min}^2}{T^{p+\alpha}} \le \mu^2 T^p + \mu N_{\min}.$$
Substituting for the value of $N_{\min}$, recalling that  $T^{\alpha/2}= c > 1$ and using the fact that $a+b \le \max\{2a,2b\}$, we get 

\begin{equation}\label{eq:pbound2}
p \ge \min \left\{k\ge 0:\frac{1}{T^{k}}\left(\frac{\sfc_2}{\mu^2}\right)^2 = 
2\max\left\{\mu^2 T^k, \frac{\sfc_2 }{\mu}\right\} \right\},
\end{equation}
where $\sfc_2$ is a constant.

\end{proof}
Generalizing the lower bound derived in Theorem \ref{thm:lbsinglearm} to allow for biased estimators
appears to be technically challenging since biased estimators with minimum variance are
difficult to find, as well as if found, tend to depend on the parameter to be estimated, making
them unrealizable \cite{kay1993fundamentals}.

 \section{Extensions to general sampling cost functions}
 \begin{remark}\label{rem:fconvex}
Extension of Theorem \ref{thm:ubsinglearm} and Theorem \ref{thm:lbsinglearm} (with the same algorithm \textsc{CTSAB} where 
the sampling frequency is chosen so as to optimize \eqref{eq:payofforaclebo}) to general convex functions $f$ for the sampling cost, other than $f(x) = 1/x$ is readily possible. 
Specifically, Prop. \ref{prop:uniformsampling} remains unchanged as long as $f$ is convex, while the optimal payoff in Prop. \ref{prop:oracle} will depend on the exact function $f$. Moreover, other arguments made to derive Theorem \ref{thm:lbsinglearm} and Theorem \ref{thm:ubsinglearm} can also be extended for general $f$, however, deriving exact expressions similar to the  lower bound Theorem \ref{thm:lbsinglearm} and upper bound Theorem \ref{thm:ubsinglearm} requires extra work, that is beyond the scope of this paper.
%where the expressions will depend on $f$. 
%However, it is worthwhile noting that finding the exact expression and the ensuing regret 
\end{remark}

% \section{}\begin{remark}\label{rem:choiceeps} 
%Recall from \eqref{defn:p},  the definition of $p^\star$.
%Let $ \mu \le \mu_{\max}$ be an upper bound on the actual value 
%of $\mu$, which is known to the algorithm.
%Correspondingly, let the algorithm know 
%of $p_{\min} \le p^\star$, where  
%$$p_{\min} = \min \left\{k\ge 0:\frac{c(1-2/T^\alpha)^2}{\mu_{\max}^4 T^{k+\alpha}} = 
%\mu_{\max}^2 T^k\right\}$$ for all 
% $\alpha>0$. Assuming $ p_{\min} >0$, a useful choice for parameter $\epsilon$ with the  \textsc{CTSAB} algorithm is 
% $\epsilon < p_{\min}$, that controls the speed of the algorithm. 
% \end{remark}

\section{Pseudo Code for Algorithm (\textsc{CTMAB})}
\begin{algorithm}
	\caption{Continuous Time Multi-Arm Bandit (\textsc{CTMAB})}
	\label{alg:CTMultiBandit}
	\begin{algorithmic}[1]
	\STATE {\bf Input } $0<\epsilon<1$,  $\kappa > 1$, $T, \delta$ 
			\STATE $\b1_{x}=1$ if $x=1$ and $0$ otherwise
			\STATE\%{\bf Estimation Period Starts}
			\FOR {phases $i=1,2,3, 	 \ldots$} 
			\STATE Obtain $N_i[j] = \kappa \log(T) T^{(2/3)i\epsilon}$ samples at uniform frequency for each arm $j, j=1, \dots, K,$ simultaneously in interval $[ T^{(i-1)\epsilon}-\b1_{i=1}, \ \ T^{i \epsilon} ]$
			\STATE After every sampling time $t$ in each phase, $N^t[j]$ is the number of samples obtained for arm $j$ till time $t$,  and empirical mean for arm $j$, 
			${\hat \mu}_t[j] = \frac{1}{N^t[j]}\sum_{k=1}^{N^t[j]} X_k$, 
	
	\IF{$\sqrt{\frac{\log(2/\delta)}{N^t[k]}} < \frac{{\hat \mu}_t[k]}{2}$ \ \text{for any arm} \ $k$ }
	\STATE Let $a_t$ be the index of the arm with the largest estimated mean  i.e. 
	\STATE $a_t= \arg \max_{j=1,\dots,K} {\hat \mu}_t[j]$, 
					\STATE Define ${\hat \mu}[1] = {\hat \mu}_t[a_t]$, set $i^\star=i$ (phase index)  and Break;
					\ENDIF
					\ENDFOR
	\STATE Let the time at which the estimation period end be  $t_{\text{Es}}$ 
	\IF{$t_{\text{Es}}  \ge T$}
				\STATE Break;
					\ELSE
		\STATE \%{\bf Begin identification period} 
		\STATE Execute Best Arm Identification using the \textsc{LUCB1} algorithm \cite{ICML12_PACSubsetMAB} with input ($K$ arms, required probability of success $1-\delta=1-1/T^2$ ), where arms to be sampled by 
		\textsc{LUCB1} are sampled {\bf at frequency one sample per $1/{\hat \mu}[1]$ time}
		\STATE Let $k^\star$ be the arm identified by the \textsc{LUCB1} as the best arm
		%\STATE {\bf Input }  $T, \beta=\Delta/2, \gamma=1/T^2$
		%\STATE Set $S=[1:K]$ \%set of arms
		%\STATE Initialize $\beta_1= \beta/4, \gamma_1 = \gamma/2,  \ell=1$,
		%\STATE 
%		$k\in S$ for $\frac{1}{(\beta_\ell/2)^2}\ln\left(\frac{3}{\gamma_\ell}\right)$ 
%		\STATE Let ${\hat \mu}[k]_\ell$ denote the empirical value of arm $k$
%		\STATE Find the median of ${\hat \mu}[k]_\ell$ for all $k\in S$ denoted by $m_\ell$
%		\STATE $S_{\ell+1} = S_\ell \backslash \{k: {\hat \mu}[k]_\ell < m_\ell\}$ \% Remove the bottom half of the arms
%		\IF{$|S_{\ell}=1|$}
%		\STATE break; 
%		\ELSE
%		\STATE $\beta_{\ell+1}= 3/4\beta_{\ell}, \gamma_{\ell+1} = \gamma_\ell/2,  \ell=\ell+1$
%		\STATE Go to Step 5.
%		\ENDIF
%		\STATE \%identification period ends
		\STATE Let Algorithm \textsc{LUCB1} terminate at time $t_{\text{Es}}  + t_{\text{Id}}$
		\IF{$t_{\text{Es}} +t_{\text{Id}} \ge T$}
				\STATE Break;
					\ELSE
		\STATE After the end of identification period at time $t_{\text{Es}} + t_{\text{Id}}$, consider only the $k^\star$-th arm
		\STATE \%{\bf Execute Exploit Period of Algorithm \ref{alg:CTBandit}}
		\STATE Let $N^{[1]}$ be the total number of samples obtained for the $k^\star$-th  arm in the estimation and identification period
		\STATE ${\hat \mu}_1 = \frac{1}{N^{[1]}}\sum_{i=1}^{N^{[1]}} X_{k^\star}(i)$ empirical average of the $k^\star$-th arm 
		\STATE Divide total remaining time interval $[t_{\text{Es}}+t_{\text{Id}},T]$ into $(T-t_{\text{Es}}-t_{\text{Id}})/T^{\nu_m}$ intervals (called phase) of width $T^{\nu_m}$ each, for some $\nu_m >0$
		
		\FOR{$r=1, 2, 	 \ldots$}
		\STATE Obtain $N_r = {\hat \mu}_rT^{\nu_m}/2$ samples from the $k^\star$-th  arm at equal intervals in
		$[t_{\text{Es}}+t_{\text{Id}}+ rT^{ \nu_m}, t_{\text{Es}}+t_{\text{Id}}+ (r+1)T^{ \nu_m}]$
		\STATE $N^r = N^{[1]} + \sum_{j\le r} N_j,{\hat \mu}_r = \frac{1}{N^r}\sum_{j=1}^{N^r} X_i$, \%update the sample average
		\IF{$t_{\text{Es}}+t_{\text{Id}}+(r+1) T^{\nu_m} \ge T$}
		\STATE Break;
		\ENDIF
		\ENDFOR
		\ENDIF
		\ENDIF
	\end{algorithmic}
	% \vspace{-0.4cm}%
\end{algorithm}

\section{}
\begin{lemma}\label{lem:estimationperiod}
 With probability at least $1-3\delta$, the estimation period of the algorithm \textsc{CTMAB} ends by time $O(T^{p^\star})$, (where $p^\star$ is as defined in \eqref{defn:p} with $\mu=\mu[1]$) and the empirical estimate ${\hat \mu}_t[a_t] $ of the chosen arm $a_t$ satisfies $\frac{2}{3}\mu[1]\le {\hat \mu}_t[a_t] \le 2\mu[1]$.
\end{lemma}
\begin{proof}
As defined before, let the true best arm be arm $1$. First, we argue about the estimation error, and then upper bound the time to end the estimation period.

Consider the earliest time $t$ at which $\sqrt{\frac{\log(2/\delta)}{N^t[k]}} < \frac{{\hat \mu}_t[k]}{2}$ is true for some arm $k$, i.e., the estimation period ends at time $t$. Recall that $a_t$ is defined as the index of the arm with the largest empirical mean at any sampling time $t$. We claim that either $a_t=k$ or 
$\sqrt{\frac{\log(2/\delta)}{N^t[a_t]}} < \frac{{\hat \mu}_t[a_t]}{2}$ is also true.
In case $a_t\ne k$, note that by definition $\frac{{\hat \mu}_t[k]}{2} \le \frac{{\hat \mu}_t[a_t]}{2}$, and hence $\sqrt{\frac{\log(2/\delta)}{N^t[a_t]}} < \frac{{\hat \mu}_t[a_t]}{2}$ also, since the number of samples obtained for each arm is the same, i.e. $N^t[k_1] = N^t[k_2]$, for any $k_1\ne k_2$. 
%Therefore  arm $a_t$ also satisfies  if for some arm $k$, $\sqrt{\frac{\log(2/\delta)}{N^t[k]}} < \frac{{\hat \mu}_t[k]}{2}$ is true for some arm $k$

Thus, we restrict our attention to arm $a_t$ alone for estimating the error $|{\hat \mu}_t[a_t] - \mu[1]|$.

Case 1:  $a_t=1$ when the estimation period terminates. Thus, $\sqrt{\frac{\log(2/\delta)}{N^t[1]}} < \frac{{\hat \mu}_t[1]}{2}$.  Using Corollary \ref{cor:conc}, we have that at any time, for arm $1$ with probability at least $1-\delta$, $|{\hat \mu}[1] - \mu[1]| \le  \sqrt{\frac{\log(1/\delta)}{N^t[1]}}$, which together with $\sqrt{\frac{\log(2/\delta)}{N^t[1]}} < \frac{{\hat \mu}_t[1]}{2}$ shows  that $|{\hat \mu}_t[1]-\mu[1]|\le {\hat \mu}[1]/2$ implying $\frac{2}{3}\mu[1]\le {\hat \mu}_t[a_t] \le 2\mu[1]$. 

Case 2:  $a_t\ne 1$ when the estimation period terminates. Let $a_t = k > 1$, and 
 $\sqrt{\frac{\log(2/\delta)}{N^t[k]}} < \frac{{\hat \mu}_t[k]}{2}$.  
   Using Corollary \ref{cor:conc}, we have that at any time, for arm $1$ with probability at least $1-\delta$, ${\hat \mu}[1] + \sqrt{\frac{\log(1/\delta)}{N^t[1]}}\ge \mu[1]$. 
 Since $a_t=k >1$, it must be that ${\hat \mu}[k] + \sqrt{\frac{\log(1/\delta)}{N^t[k]}} > \mu[1]$ with probability at least $1-\delta$, since with probability at least $1-\delta$, ${\hat \mu}[1] + \sqrt{\frac{\log(1/\delta)}{N^t[1]}}\ge \mu[1]$. Thus, for arm $k$, when
  $\sqrt{\frac{\log(2/\delta)}{N^t[k]}} < \frac{{\hat \mu}_t[k]}{2}$, it implies that 
 ${\hat \mu}[k] + \frac{{\hat \mu}_t[k]}{2} > \mu[1]$ which means that 
 ${\hat \mu}_t[k] \ge \frac{2}{3}\mu[1]$. As in Case 1,  $\sqrt{\frac{\log(2/\delta)}{N^t[k]}} < \frac{{\hat \mu}_t[k]}{2}$  implies that 
  ${\hat \mu}_t[k]\le2 \mu[k]$ with probability at least $1-\delta$. Since arm $1$ is the true best arm $\mu[k] \le \mu[1]$,  hence we can conclude that ${\hat \mu}_t[k]\le2 \mu[1]$. 
 
% Thus, even if for some other arm $k>1$, $a_t = k$, and 
% $\sqrt{\frac{\log(2/\delta)}{N^t[k]}} < \frac{{\hat \mu}_t[k]}{2}$, $\frac{2}{3}\mu[1]\le {\hat \mu}_t[k] \le \frac{3}{2}\mu[1]$. Since for the best arm, with high probability (at least $1-\delta$) by time $T^{p^\star}$, $\sqrt{\frac{\log(2/\delta)}{N^t[1]}} < \frac{{\hat \mu}_t[1]}{2}$ for some other arm to satisfy the condition will take at most time $T^{p^\star}$.

{\bf Time for the estimation period to terminate. }
From Lemma \ref{lem:ublearningperiod}, we know that with high probability (at least $1-\delta$) by time $T^{p^\star}$ ($p^\star$ defined in \eqref{defn:p} with $\mu=\mu[1]$), $\sqrt{\frac{\log(2/\delta)}{N^t[1]}} < \frac{{\hat \mu}_t[1]}{2}$ for the true best arm, arm $1$. Since all arms are sampled simultaneously at equal frequency, thus, the estimation period ends by time at most $T^{p^\star}$ with probability $1-\delta$.

Taking the union over all the bad events (at most three of them each with probability at most $\delta$), claim holds with probability at least $1-3\delta$.
\end{proof}

\section{Bounding the payoff of the estimation period of CTMAB}
\begin{lemma}\label{lem:payoff1ctmab}
  The total payoff of algorithm \textsc{CTMAB}  in the estimation period is  
 $$	P_L \stackrel{(a)}\ge  \mu[1](\kappa K \log (T)) T^{(2/3) i^\star \epsilon } - (\kappa K \log (T))^2 T^{(i^\star+1)\epsilon/3}\stackrel{(b)}\ge -O\left(\frac{ K^2}{\mu[1]} \log^2T\right).$$
  with probability at least $1-\delta$, where $i^\star$ is the phase in which the estimation period terminates.
\end{lemma}
\begin{proof}
  The number of samples obtained in each phase of the estimation period of algorithm \textsc{CTMAB} is $K$ times the number of samples obtained in the learning period of  \textsc{CTSAB} algorithm. 
  Thus, $(a)$ directly follows from Lemma \ref{lem:payoff1}. Use Lemma \ref{lem:ublearningperiod} to get $(b)$ that shows that 
$i^\star \le \left(\frac{p^\star}{\epsilon}+1\right)$ 
with probability at least $1-\delta$ and  the definition of $p^\star$ \eqref{defn:p} with $\mu=\mu[1]$.
\end{proof}

%By definition \eqref{eq:payoffA}, the payoff for any algorithm in any time period is lower bounded by $0$. For the proposed algorithm, however, we are deriving lower bound on the payoff which are negative without considering the $\max\{., 0\}$ to get the worst case results.

\section{Bounding the completion time of  the identification period of algorithm  \textsc{CTMAB}}	\label{app:ubm1}

Using Lemma \ref{lem:ubpac} we bound the time needed for the completion of the identification period of  algorithm \textsc{CTMAB} in Lemma \ref{lem:idphasetime}.
\begin{lemma}\label{lem:ubpac}\cite{ICML12_PACSubsetMAB}[Corollary 7]   The \textsc{LUCB1} algorithm needs $O\left(\frac{K\log T}{\Delta^2}\right)$  samples to identify the 
best arm with high probability $1-\delta = 1 - 1/T$.\end{lemma}

\begin{lemma}\label{lem:idphasetime} With probability at least $1-4 \delta$, the identification period of algorithm  \textsc{CTMAB}  is complete by time $O\left(\frac{4K}{\mu[1]\Delta^2} \log \left(T^2\right) \right)$.
%$$\log \left(T^2\right) \log \left(\frac{1}{\Delta^2}\right)\le O(T^{\epsilon_m})$$ and $p_m^\star$ is defined in \eqref{defn:pm}.
\end{lemma}

\begin{proof}
From Lemma \ref{lem:estimationperiod}, at the end of estimation period of algorithm  \textsc{CTMAB}, we know that $\frac{2}{3}\mu[1]\le {\hat \mu}_t[a_t] \le 2\mu[1]$ with  probability at least $1-3\delta$.  Since we are bounding order-wise, we next substitute ${\hat \mu}[a_t]$ with 
the true mean $\mu[1]$, as the loss compared to ${\hat \mu}[a_t]$ is only a constant.
From Lemma \ref{lem:ubpac}, we know that the \textsc{LUCB1} algorithm identifies the best arm with probability $1-\delta$ if it obtains $O\left( \frac{4K}{\Delta^2} \log \left(T^2\right) \right)$ samples. Since the \textsc{LUCB1} algorithm is executed at frequency one sample per time period $1/\mu[1]$,
%Median Elimination Algorithm is $(\beta, \gamma)$-correct with the number of samples obtained is $O\left( \frac{K}{\beta^2} \log \left(\frac{1}{\gamma}\right)\right)$. Thus, for our purposes, we choose $\beta=\Delta/2$, which will ensure that only the best arm belongs to the set of $\Delta/2$-optimal arms. Moreover, we 
%choose $\gamma=1/T^2$ to have the success probability for the Median Elimination Algorithm at least $1-1/T^2$. Therefore, the number of samples needed by the 
%Median Elimination Algorithm is $O\left( \frac{4K}{\Delta^2} \log \left(T^2\right)\right)$, which are obtained at frequency one sample per unit time. 
in time $t_{\text{Id}}=O\left(\frac{4K}{\mu[1]\Delta^2} \log \left(T^2\right) \right)$ all the required samples (needed by \textsc{LUCB1} algorithm to identify the best arm with probability $1-\delta$) have been obtained and the best arm has been identified with probability at least $1-4\delta$ (using the union bound). 
%From the definition of $p^\star_m$ \eqref{defn:pm}, it is clear that $ \frac{\mu[1]4K}{\Delta^2} = O(T^{p^\star_m})$.
 %Therefore, the identification period completion time $t_{\text{Id}} = O(T^{p^\star_m}\log T)$ with high probability.
 \end{proof}
 %	XXX $ \frac{4K}{\Delta^2} \log \left(T^2\right)= O(\mu[1] T^{p^\star_m} \log T)$ isnt it?
	% $\mu[1]^2$ should be replaced by $\mu[1]$.

\section{Proof of Thm.~\ref{thm:ubmultiarm} }	\label{app:ubm2}
%Recall that the number of samples obtained  in the identification period is $N^I = O\left( \frac{1}{\Delta^2} \log \left(T^2\right)\right)$ to guarantee that the best arm is identified with probability at least $\left(1-\frac{1}{T^2}\right)$, 
With $\delta=\frac{1}{T^2}$, from Lemma \ref{lem:ubpac}, the \textsc{LUCB1} algorithm obtains $N_{\text{Id}}= O\left( \frac{4K}{\Delta^2} \log \left(T^2\right) \right)$ samples in the identification period whose time duration is $t_{\text{Id}} =O\left(\frac{4K}{\mu[1]\Delta^2} \log \left(T^2\right) \right)$ from Lemma \ref{lem:idphasetime}, and succeeds with probability $1-\frac{4}{T^2}$ in identifying the best arm. For the rest of the proof, we assume that the best arm has been identified in the identification period, and account for its probability correspondingly. The payoff obtained in the identification period of the   \textsc{CTSAB} algorithm is 

	\begin{align}\label{eq:payoffmultiarm1}
	P_{\text{Id}} & \stackrel{(a)}=\left[ \sum_{k=1}^K \mu[k]N^{[k]}-  N_{\text{Id}}^2/t_{\text{Id}}\right]  \stackrel{(b)}\ge -   N_{\text{Id}}^2/t_{\text{Id}} \stackrel{(c)}=  -O\left(\frac{4\mu[1]K}{\Delta^2} \log \left(T^2\right) \right),
	\end{align}
	where in $(a)$, $N^{[k]}$ is the number of samples obtained for arm $k$ in the identification period, to get $(b)$, we ignore the first positive term of $(a)$, and for $(c)$ we use the definition of $t_{\text{Id}} = O\left(\frac{4K}{\mu[1]\Delta^2} \log \left(T^2\right) \right)$.

	%XXX : $\mu[1]^2$ should be replaced by $\mu[1]$.
	
	%and finally for $(d)$, we use the definition of $p^\star_m$ \eqref{defn:pm} which implies that $\frac{K}{\Delta^2} = O\left(\mu[1]^2T^{p^\star_m}\right)$ and the condition that $\log \left(T^2\right) \le O(T^{\epsilon_m})$ needed for Lemma \ref{lem:idphasetime}, for upper bounding the first term of $(c)$. 
	
	For further exposition, we  need the following Corollary of Lemma \ref{lem:ubpac}.
	\begin{corollary}\label{cor:ubpac}
With $\delta=\frac{1}{T^2}$, the number of samples obtained by the \textsc{LUCB1} algorithm for the best arm is at least 
$$\left(\frac{4c}{\Delta^2} \log \left(T^2\right)\right) $$ for some constant $c$, with probability at least $1-\frac{1}{T^2}$.
\end{corollary}

%	Note that since $\Delta < \mu[1]$, we have that $p^\star_m \ge p^\star$ from the definition of \eqref{defn:pm} and \eqref{defn:p}. Moreover, from Corollary \ref{cor:ubpac}, the number of samples $N^{[1]}$ obtained  
%	 for the best arm  in the identification period of the  \textsc{CTSAB} algorithm is $$\left(\frac{4c}{\Delta^2} \log \left(T^2\right)\right) \ge \left( \frac{4c}{\mu[1]^2} \log \left(T^2\right)\right)$$ for some constant $c$ with probability at least $1-1/T^2$. 

Let ${\hat \mu[1]}_t$ be the empirical estimate of the mean of the best arm at the end of identification period of  algorithm \textsc{CTMAB}, i.e., at time $t$ after obtaining $N^{[1]}$ samples for the best arm.
As before, let $\text{err}_t = |{\hat \mu[1]}_t-\mu[1]|$. From Corollary \ref{cor:ubpac}, the number of samples $N^{[1]}$ obtained  
	 for the best arm  in the identification period of the  \textsc{CTSAB} algorithm is $$\left(\frac{4c}{\Delta^2} \log \left(T^2\right)\right) \ge \left( \frac{4c}{\mu[1]^2} \log \left(T^2\right)\right)$$ for some constant $c$ with probability at least $1-1/T^2$.

	 Thus, we have from 
	Corollary \ref{cor:conc} and Lemma \ref{lem:direct}, that with probability at least $1-1/T^2$, $\text{err}_t < \mu[1]$. Thus, similar to the single arm problem case \eqref{eq:payoffwitherr}, payoff for phase $r=1$ of the  \textsc{CTSAB} algorithm after the identification period is	\begin{align}
	P_1&= \frac{\mu[1]^2T^{\nu_m}}{2} - \text{err}_t^2\frac{T^{\nu_m}}{4} > 0.
	\end{align}
	Therefore, for each phase $r\ge 1$ of  the  \textsc{CTSAB} algorithm that starts after the identification period, the payoff $P_r$ is same as in \eqref{eq:dummy900}, considering the total time horizon as $T-t_{\text{Es}}-t_{\text{Id}}$ with each phase width of $T^{\nu_m}$, and 
	assuming that the best arm identified in the identification period is in fact 
	arm $1$ (which happens with probability $\left(1-\frac{1}{T^2}\right)$). 
	
	Therefore, using \eqref{eq:dummy900},  similar to \eqref{eq:finalpayoffsinglearm}, the payoff of the  \textsc{CTMAB} algorithm obtained using the exploit period of algorithm \textsc{CTSAB} with the best identified arm as the single arm,
	\begin{align}\nn
	P_{\text{single},\opt}([t_{\text{Es}}+ t_{\text{Id}}, T])  & = \sum_{r=1}^{\frac{T-t_{\text{Es}}-t_{\text{Id}}}{T^{\nu_m}}}P_r,\\\nn
	& \ge     \frac{\mu[1]^2(T- t_{\text{Es}}-t_{\text{Id}})}{4}-\frac{\ln(2/\delta)}{4\mu[1]  } \ln \left(\frac{T}{T^{\nu_m}}\right), \\\nn
	& \stackrel{(a)}\ge     \frac{\mu[1]^2\left(T- O\left(\frac{K^2\log^2 T}{\mu[1]}\right)- O\left(\frac{4K}{\mu[1]\Delta^2} \log \left(T^2\right) \right)\right)}{4}-\frac{\ln(2/\delta)}{4\mu[1]  } \ln \left(\frac{T}{T^{\nu_m}}\right), \\ \nn
	& \stackrel{(b)}=
	\frac{\mu[1]^2 T}{4}  -O\left(\mu[1]K^2\log^2 T\right)- O\left(\frac{4\mu[1]K}{\Delta^2} \log \left(T^2\right) \right)  \\ \nn
	& \quad - O\left(\frac{1}{\mu[1]} \log (T^2) \log (T)\right),\\ \label{eq:finalpayoffmultiarm}
	& \stackrel{(c)}\ge \frac{\mu[1]^2 T}{4}  -O\left(\max\left\{\frac{\mu[1]K}{\Delta^2}\log T, \mu[1]K^2\log^2 T, \frac{\log^2 T}{\mu[1]} \right\}\right),
	\end{align}
	with probability at least
	$\left(1-\frac{T}{T^{2+\nu_m}} -\frac{5}{T^2}\right)$ with $\delta=\frac{1}{T^2}$, where in $(a)$ we have used $t_{\text{Es}} = O\left(\frac{K^2\log^2 T}{\mu[1]}\right)$ from Lemma \ref{lem:estimationperiod}, while
	 $t_{\text{Id}} = O\left(\frac{4K}{\mu[1]\Delta^2} \log \left(T^2\right) \right)$ from Lemma \ref{lem:idphasetime}, and $T^{\nu_m}$ is the width of each phase after the identification period. 
	 
	 Let $\delta' = \frac{T}{T^{2+\nu_m}} -\frac{5}{T^2}$.

	Incorporating the payoff of \textsc{CTMAB} algorithm for the estimation period from Lemma \ref{lem:payoff1ctmab} and the identification period from \eqref{eq:payoffmultiarm1}, the total payoff of the  \textsc{CTMAB} algorithm is 
	\begin{align}\nn
	P& =P([0,t_{\text{Es}}])+ P([t_{\text{Es}}, t_{\text{Es}}+ t_{\text{Id}})+ P_{\text{single},\opt}([t_{\text{Es}}+ t_{\text{Id}}, T]), \\   \nn
& \ge -O\left(\frac{ K^2}{\mu[1]} \log^2T\right) -O\left(\frac{4\mu[1]K}{\Delta^2} \log \left(T^2\right) \right) \\ \nn
& \quad \quad
 + \frac{\mu[1]^2 T}{4}  -O\left(\max\left\{\frac{\mu[1]K}{\Delta^2}\log T, \mu[1]K^2\log^2 T,\frac{\log^2 T}{\mu[1]}\right\}\right)
 \end{align}
	with probability at least
	$1-\delta'$.
	%Choosing $\delta = 1/T^2$, the overall regret for the 
	%algorithm with multiple arms is at most 
	%\begin{align}
	%    R &\le (\log (T))^2 O(\mu[1]^2T^{p})+ R_{\text{single}},
	%\end{align}
	%with probability 
	%$\left(1-\frac{T}{T^{\epsilon}}\delta\right)
	%(1-\delta_I)=
	%\left(1-\frac{T}{T^{\epsilon}}\frac{1}{T^2}\right)
	%\left(1-\frac{1}{T}\right)$.
	%Choosing $\delta = 1/T^2$, and since the 
	%maximum regret can be $\mu^2T/2$, we get that the expected regret 

	Recall that the payoff of oracle policy is  
	$\frac{\mu[1]^2 T}{4}$. 
	Thus,  the 
	regret of algorithm \textsc{CTMAB} is at most 	\begin{align}\nn
	R&\le \frac{\mu[1]^2 T}{4} - P, \\\nn
	&\le O\left(\max\left\{\frac{\mu[1]K}{\Delta^2}\log T, \frac{K^2}{\mu[1]}\log^2 T, \frac{\log^2 T}{\mu[1]} \right\}\right) 
	\end{align}
	with probability $1-\delta'$.

	Since the maximum regret of algorithm \textsc{CTMAB} can be $O(T)$ (Remark \ref{rem:univregretbound}), we get that the expected regret of algorithm \textsc{CTSAB} is at most 
	\begin{align*}
	\cR = {\mathbb E}\{R\} & =  O\left(\max\left\{\frac{\mu[1]K}{\Delta^2}\log T, \frac{K^2}{\mu[1]}\log^2 T, \frac{\log^2 T}{\mu[1]} \right\}\right)
	(1-\delta') + O(T)\delta'.
	%& \le \max\left\{O\left(\frac{\mu[1]K\log T}{\Delta^2 \lambda}\right), O\left(\frac{\lambda K^2}{\mu[1]} \log^2 (T)\right), O\left(\frac{\lambda \mu[1]^2K\log T}{\Delta^2 \lambda}\right)\right\}\\
	%& \quad\left(1-\frac{T}{T^{\nu_m}}\frac{1}{T^2}\right) \left(1-\frac{1}{T^2}\right) +o(T),
	\end{align*}

\section{Preliminaries to prove  Theorem~\ref{thm:lbmultiplearms}}
{\bf Exploration Problem with K arms of unknown means:} Let there be $K$ arms, with i.i.d. Bernoulli distribution $\cP_k$ for arm $k$ with mean $\mu[k]$ and $1>\mu[1] > \dots > \mu[K]$. Let the product distribution over the $K$ arms be denoted as 
$\cP = \prod_{i=1}^K\cP_i$.
An arm $k$ is called $\beta$-optimal if $\mu[k] \ge \mu[1]-\beta$. 
An algorithm is $(\beta,\gamma)$-correct if for the arm $k^\star$ that it outputs as the best arm, we have 
$\bbP_{\cP}\left(\mu[k^\star] \ge \max_{1\le k\le K} \mu[k]-\beta \right) \ge 1-\gamma$.

%An algorithm is an $(\beta,\gamma)$-correct algorithm if it outputs an $\beta$-optimal arm with probability at least 
%$1-\gamma$ when it terminates. More formally,
%let an algorithm declare arm $k^\star$ as the best arm after obtaining $N$ samples. Such an algorithm is called $(\beta,\gamma)$-correct if 
%$\bbP_{\cP}\left(\mu[k^\star] \ge \max_{1\le k\le K} \mu[k]-\beta \right) \ge 1-\gamma$.
\begin{lemma}\label{lem:lbpureexplore} \cite{mannor2004sample}
There exist positive constants $c_1, c_2, \beta_0, \gamma_0$ such that for $K\ge 2, \beta \in (0,\beta_0)$ there exists distribution $\cP$ such that for  every
$\gamma \in (0,\gamma_0)$, the sample complexity of any  $(\beta,\gamma)$-correct algorithm is at least 
$N \ge c_1 \frac{K}{\beta^2} \log \left(\frac{c_2}{\gamma}\right)$, where $\beta_0 = 1/8$ and $\delta_0 = e^{-4}/4$.
\end{lemma}

Using Lemma \ref{lem:lbpureexplore}, next, we derive a lower bound on the regret of any online algorithm for the CTMAB problem. 
%We will characterize the regret of any algorithm as $\mu^2T^{p_m}$ for some $p_m \in(0,1)$, where this choice is driven by the oracle payoff being $\mu^2 T/4$. Note that $p_m$ can depend on any parameter of the problem, e.g., $\mu, \Delta, T$, hence, this does not limit the generality. 
%Next, given that an algorithm has regret $ \mu^2T^{p_m}$, we will derive a lower bound on $p_m$.
%We will use Lemma \ref{lem:lbpureexplore} with $\beta=\Delta/2$ so that the set of $\beta$-optimal arms only contains the best arm, to derive the following result.

\section{Proof of Theorem.~\ref{thm:lbmultiplearms}}
Consider any online algorithm $\cA$ for the multiple arms CTMAB that has regret 
$g_m(\mu[1], \dots, \mu[K], \Delta, T)$ which can be expressed as $g_m(\mu[1], \dots, \mu[K], \Delta, T)= \mu^2 T^{p_m}$ for some $p_m<1$ that can depend on $\mu[1], \dots, \mu[K], \Delta, T$.

	We divide the total time horizon $T$ into two intervals, first $[0, c_{12}T^{p_m}]$, and second  $[c_{12}T^{p_m}, T]$ where $c_{12} = \frac{\mu[1]^2}{\mu[2]^2}$. Then the payoff of algorithm $\cA$ can be written as  
 	\begin{align}\nn P_\cA &=P_\cA([0,c_{12}T^{p_m}]) + P_\cA([c_{12}T^{p_m}, T]), \\ \label{eq:multiarmarmpayoffdummy1}
 	&\stackrel{(a)} \le P_\cA([0,c_{12}T^{p_m})]  +  \frac{\mu[1]^2 (T-c_{12}T^{p_m})}{4},
 	\end{align}
 	where in $(a)$
 	for the second interval $[c_{12}T^{p_m}, T]$, we assume that $\cA$
 	knows the index of the best arm and the true value of $\mu[1]$ at time $c_{12}T^{p_m}$, and obtains the payoff  equal to that of the oracle policy for interval $[c_{12}T^{p_m}, T]$.

	Let at time $t$, arm $k^\star(t,N_t)$ be the arm identified by $\cA$ as the best arm.
	Let $\delta_I(t,N_t)=\bbP(k^\star(t,N_t) \ne 1)$ be the probability with which $\cA$ mis-identifies the best arm at time $t$ after obtaining $N_t$ samples.
	For $\psi>0$, such that $T^\psi < c_{12}$, let $N_{\min}^m(\psi) = \frac{\sfc_4 K}{(\Delta/2)^2}
 	\log\left(\sfc_5\frac{c_{12}}{8T^\psi }\right)$.
	Let at time $c_{12}T^{p_m}$,  the total number of samples obtained by $\cA$ be 
	$$N_\cA(c_{12}T^{p_m}) < N_{\min}^m(\psi).$$ Consequently, letting $\beta = \Delta/2$ (which ensures only the best arm is part of the $\beta$-optimal arm set), from Lemma \ref{lem:lbpureexplore}, we get that $\delta_I (c_{12}T^{p_m}, N_\cA(c_{12}T^{p_m}))\ge   \frac{T^{\psi}}{c_{12}}$.

%	Then the payoff of algorithm $\cA$ is
%	
%	For $\cA$, let $k^\star(t, N_t)$ be its estimate of the index of the arm with the largest mean at time $t$ after obtaining $N_t$ samples in time 
% 	interval $[0,t]$. Let $\delta_I(t,N_t)=\bbP(k^\star(t,N_t) \ne 1)$ be the probability with which $\cA$ mis-identifies the best arm at time $t$ with $N_t$ samples.
% 	
% 	We divide the total time horizon $T$ into two intervals, first $[0, c_{12}T^{p_m}]$, and second  $[c_{12}T^{p_m}, T]$ for some $\alpha>0$ to be chosen later. For the first interval, out of the total 
% 	$N_\cA(c_{12}T^{p_m})$ samples that $\cA$ obtained uniformly (the best case) over $[0, c_{12}T^{p_m}]$, let $N_\cA(c_{12}T^{p_m})[k]$ be the samples obtained for arm $k$ such that $\sum_{k=1}^K N_\cA(c_{12}T^{p_m})[k] = N_\cA(c_{12}T^{p_m})$. 
%	
	
 	Recall that the maximum payoff obtainable by the oracle policy in time interval $[0,t]$ from arm $k$  is $\frac{\mu[k]^2 t}{4}$ and across all arms is $\frac{\mu[1]^2 t}{4}$.
 	Hence, 
 	 the payoff of $\cA$ for the first interval
	 \begin{align}\label{eq:dummyxx3} P_\cA([0,c_{12}T^{p_m})] \le & (1-\delta_I (c_{12}T^{p_m}, N_\cA(c_{12}T^{p_m}))) \frac{\mu[1]^2 c_{12}T^{p_m}}{4}  + \delta_I (c_{12}T^{p_m}, N_\cA(c_{12}T^{p_m})) \frac{\mu[2]^2 c_{12}T^{p_m}}{4}, \end{align} 
	 where to get the second term we are upper bounding $\mu[k] < \mu[2]$ whenever $k^\star(t,N_t)>2$.
	 Since $\delta_I (c_{12}T^{p_m}, N_\cA(c_{12}T^{p_m}))\ge \frac{T^{\psi}}{c_{12}} $, from \eqref{eq:dummyxx3} we get 
 \begin{align}\label{eq:upmdum1}
 P_\cA([0,c_{12}T^{p_m})]&\le \frac{\mu[1]^2 c_{12}T^{p_m}}{4} -  \frac{\mu[1]^2 T^{p_m+\psi}}{4} +  \frac{\mu[2]^2 c_{12} T^{p_m}}{4},
 \end{align}
 where to get the third term we have just used the trivial bound $\delta_I (c_{12}T^{p_m}, N_\cA(c_{12}T^{p_m})) \le 1$.
 Combining  \eqref{eq:upmdum1} with \eqref{eq:multiarmarmpayoffdummy1},  and using the definition of $c_{12}$, 
 we get that 
 \begin{align}\nn P_\cA &\le  \frac{\mu[1]^2T}{4} -  \frac{\mu[1]^2}{4}(T^{p_m+\psi}) +  \frac{\mu[1]^2}{4}(T^{p_m}),
 	\end{align}
	making the regret of $\cA$ more than $\frac{\mu[1]^2}{4}(T^{p_m})$, giving a contradiction. Thus, it is necessary that $N_\cA(c_{12}T^{p_m}) < N_{\min}^m(\psi)$ for all $\psi>0$. Using the definition of $N_{\min}^m(\psi)$ we get 
	$$N_\cA(c_{12}T^{p_m}) >\frac{\sfc_4 K}{(\Delta/2)^2}
 	\log\left(\sfc_5\frac{c_{12}}{8T^\psi }\right).$$
	Next, letting $\psi$ small, hence we get that 
	\begin{equation} \label{eq:dummyxx4} N_\cA(c_{12}T^{p_m}) > \frac{\sfc_4 K}{(\Delta/2)^2}
 	\log\left(\sfc_5\frac{c_{12}}{8}\right) = N_{\min}^m
	\end{equation}

	The payoff of $\cA$ in the first interval $[0,c_{12}T^{p_m})]$ as a function of the number of samples $x$ obtained by it  in time $[0, c_{12}T^{p_m}])$ for arm $1$ is
	$$P_\cA([0,c_{12}T^{p_m}]) \le \max_x \left\{ \mu[1] x -   \frac{ x^2}{c_{12}T^{p_m}} \right\}.$$ 
	\begin{remark}\label{rem:unimod}
  Note that 
	$\mu[1] x -   \frac{ x^2}{c_{12}T^{p_m}}$ is a concave and unimodal function of $x$ with optimal $x_m^\star =  \frac{\mu[1] c_{12}T^{p_m}}{2}$. 
\end{remark}
	
	Thus, we consider two cases : i) $x_m^\star \ge N_{\min}^m$ and ii) $x_m^\star  <N_{\min}^m$, where $N_{\min}^m$ has been defined in \eqref{eq:dummyxx4}. When  $x_m^\star \ge N_{\min}^m $, using the expressions for $x_m^\star$ and $N_{\min}^m $,  we get 
that $\frac{\mu[1] c_{12}T^{p_m}}{2}  \ge \frac{\sfc_4 K}{(\Delta/2)^2}
 	\log\left(\sfc_5\frac{c_{12}}{8}\right)$, implying 
\begin{equation}\label{eq:pmbound1}
c_{12}T^{p_m} \ge \frac{2 \sfc_4 K}{(\Delta/2)^2 \mu[1]}
 	\log\left(\sfc_5\frac{c_{12}}{8}\right) \quad \forall \ \alpha>0.
\end{equation}
Thus, we get that $T^{p_m} = \Omega\left(\frac{K}{\Delta^2\mu[1]}\right)$.

For case ii) $x_m^\star  < N_{\min}^m$, we proceed as follows. 
Writing out the expected payoff of $\cA$ in the first interval as the function of samples obtained, 
	\begin{align}\label{eq:dummy449}
\nn \bbE\{P([0,c_{12}T^{p_m})]\} &\le \mu[1] N_\cA(c_{12}T^{p_m}) -  \\ & \frac{N_\cA(c_{12}T^{p_m})^2}{c_{12}T^{p_m}},
\end{align} since arm $1$ is the best.

We have already argued that $N_\cA(c_{12}T^{p_m})\ge N_{\min}^m$. Thus, following Remark \ref{rem:unimod}, with $x^\star  < N_{\min}^m$, 
it is clear that the RHS of \eqref{eq:dummy449}  is a decreasing function of $N_\cA(c_{12}T^{p_m})\ge N_{\min}^m$ for fixed $\mu[1]$ and $T$. Thus, choosing $N_\cA(c_{12}T^{p_m}) = N_{\min}^m$ (the minimum possible), from \eqref{eq:multiarmarmpayoffdummy1}, we obtain the largest payoff of algorithm $\cA$ for the first interval, and the overall payoff \eqref{eq:multiarmarmpayoffdummy1} over the two intervals is 
\begin{align} P &\le \mu[1] N_{\min}^m -   \frac{(N_{\min}^m)^2}{c_{12}T^{p_m}}    +  \frac{\mu[1]^2 (T-c_{12}T^{p_m})}{4}.\end{align}

Thus, for the regret of algorithm $\cA$ to be $\mu[1]^2 T^{p_m}$, we need $$\frac{  (N_{\min}^m)^2}{c_{12}T^{p_m}} \le \mu[1]^2 T^{p_m} + \mu[1] N_{\min}^m.$$
Substituting for the value of $N_{\min}$, we get 

\begin{align}\label{eq:pmbound2}
p_m \ge \min \left\{k\ge 0: \frac{ \left(\frac{\sfc_4 K}{(\Delta/2)^2}
 	\log\left(\sfc_5\frac{c_{12}}{8}\right)\right)^2}{c_{12}T^{k}}  = 
\max\left\{\mu[1]^2 T^k, \mu[1]\frac{\sfc_4 K}{(\Delta/2)^2}
 	\log\left(\sfc_5\frac{c_{12}}{8}\right)\right\} \right\}.
\end{align}
%XXX in the above equation in the first term the denominator should be $T^{k+\alpha}$.
Combining \eqref{eq:pmbound1} and \eqref{eq:pmbound2}, when $\Delta$ is small, i.e. $\mu[1] \approx \mu[2]$, we get that $T^{p_m} = \Omega\left(\frac{ K}{\Delta^2\mu[1]}\right)$.

\section{Numerical Results for the single arm case}
In this section, we present the comparitive performance of our algorithms against the oracle policy and a baseline policy that does not adapt to the estimates of the arm means for the single arm case. The multiple arm case has already been discussed in the main body of the paper. The baseline policy samples the optimal arm at a fixed  interval of $1/aT$, where $a>0$ is a constant that determines the rate of sampling. The payoff of the baseline policy over a period $T$ is $a T(\mu-\lambda a)$, and  the payoff is positive and increasing for all $a \leq \frac{\mu}{2\lambda}$ achieving maxima at $a= \frac{\mu}{2\lambda}$ . We compare the performance of algorithm {\textsc CTSAB} against this baseline policy and the oracle policy for the single arm case in Fig.~\ref{fig:PayoffComparison} for different values of $\mu \in \{0.3,0.15,0.05\}$. We set the parameters to satisfy the relations $\frac{1}{\mu^3}=O(\mu T^p)$ as discussed in Theorem. \ref{thm:lbsinglearm}.  The total payoff of each policy is obtained by average over $50$ independent runs and each run is of $T=60000$ rounds. In these experiments, we set $\epsilon=0.05$ and $\delta=0.05$. In Figs. \ref{fig:SingleArm1} and \ref{fig:SingleArm2}, the confidence intervals are small and are not clearly visible, while  they are easily discernible in Fig. \ref{fig:SingleArm3}. 

Note that only total (mean) payoff of algorithm {\textsc CTSAB} is stochastic due to its adaptation to the observations while others are deterministic. The confidence interval on the  total payoff of algorithm {\textsc CTSAB} is as shown in the figures. The baseline policies corresponds to the case where the sampling is uniform at rate $a_1=0.12/(2\lambda)$ and $a_2=0.09/(2\lambda)$, named BaseLine1 and BaseLine2 in the figures. As seen from Fig. \ref{fig:PayoffComparison}, the total reward from algorithm {\textsc CTSAB} is close to optimal and better than the base line policies. For $\mu=0.05$, the cumulativereward for the Baseline1 is negative as $a_1> 0.05/\lambda$. Further, note that the gap between the total payoff of the oracle policy and algorithm {\textsc CTSAB} is increasing as $\mu$ decreases. This is natural as the learning problem gets harder as $\mu$ becomes smaller. This is explicitly depicted in Fig. \ref{fig:SingleArmDiff}, where we plot the regret of algorithm {\textsc CTSAB} as a function of $\mu$ with confidence intervals. As seen, the regret has an inverse relationship with $\mu$ in agreement with Theorem ~\ref{thm:lbsinglearm}.
%\begin{figure}
%\centering
%\includegraphics[scale=.5]{SingleArm}
%\caption{Payoff comparison of CTSAB  algorithm with the oracle policy, and the baselines for different values of $\mu$ for the single arm case. We set $\epsilon=0.05, \lambda=0.1, T=60000$.}
%\label{fig:SingleArm}
%\end{figure}

\begin{figure*}
	\centering
	\subfloat[]{\label{fig:SingleArm1}
		\includegraphics[scale=0.28]{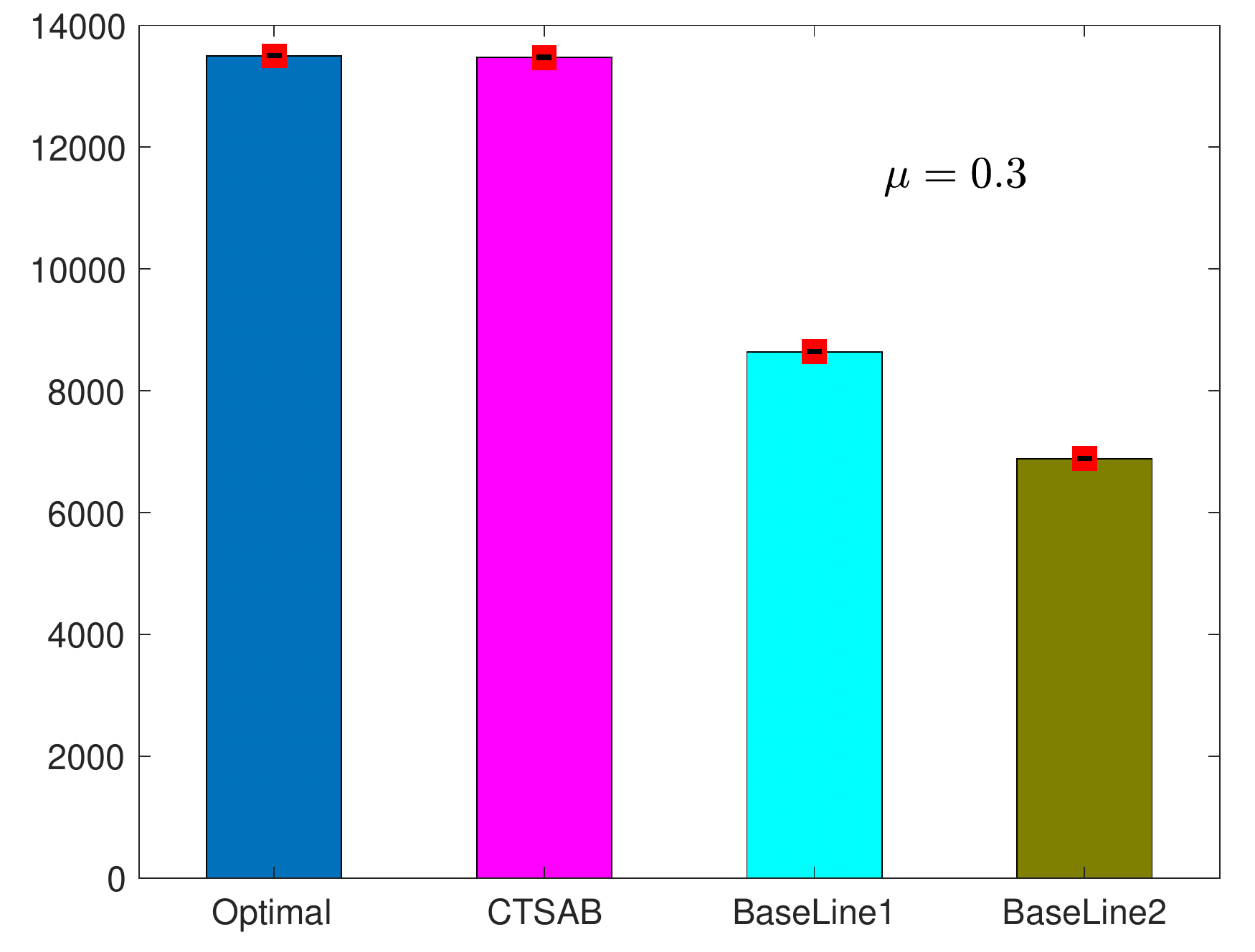}}
	\subfloat[]{\label{fig:SingleArm2}
		\includegraphics[scale=0.28]{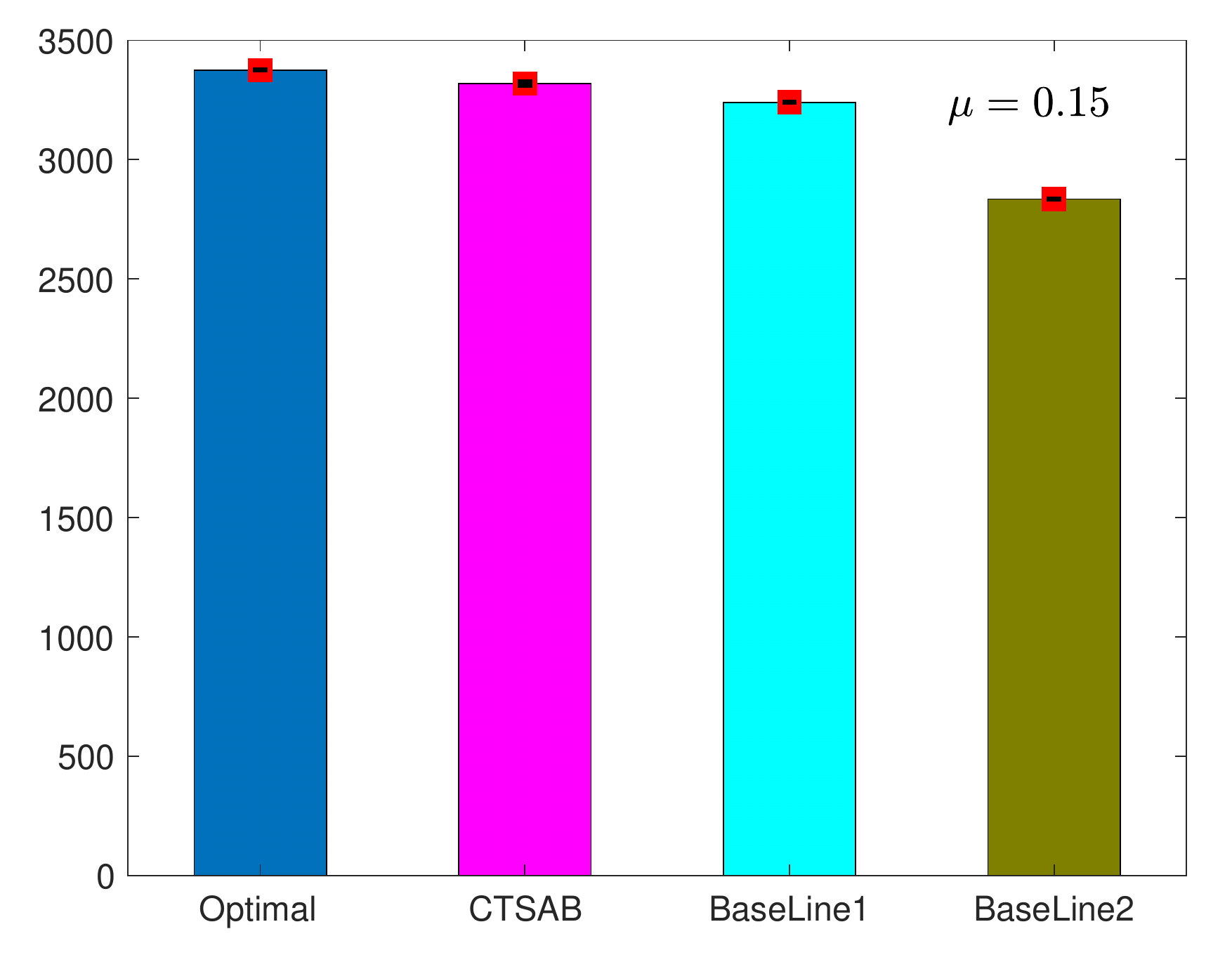}}
	\subfloat[]{\label{fig:SingleArm3}
		\includegraphics[scale=0.28]{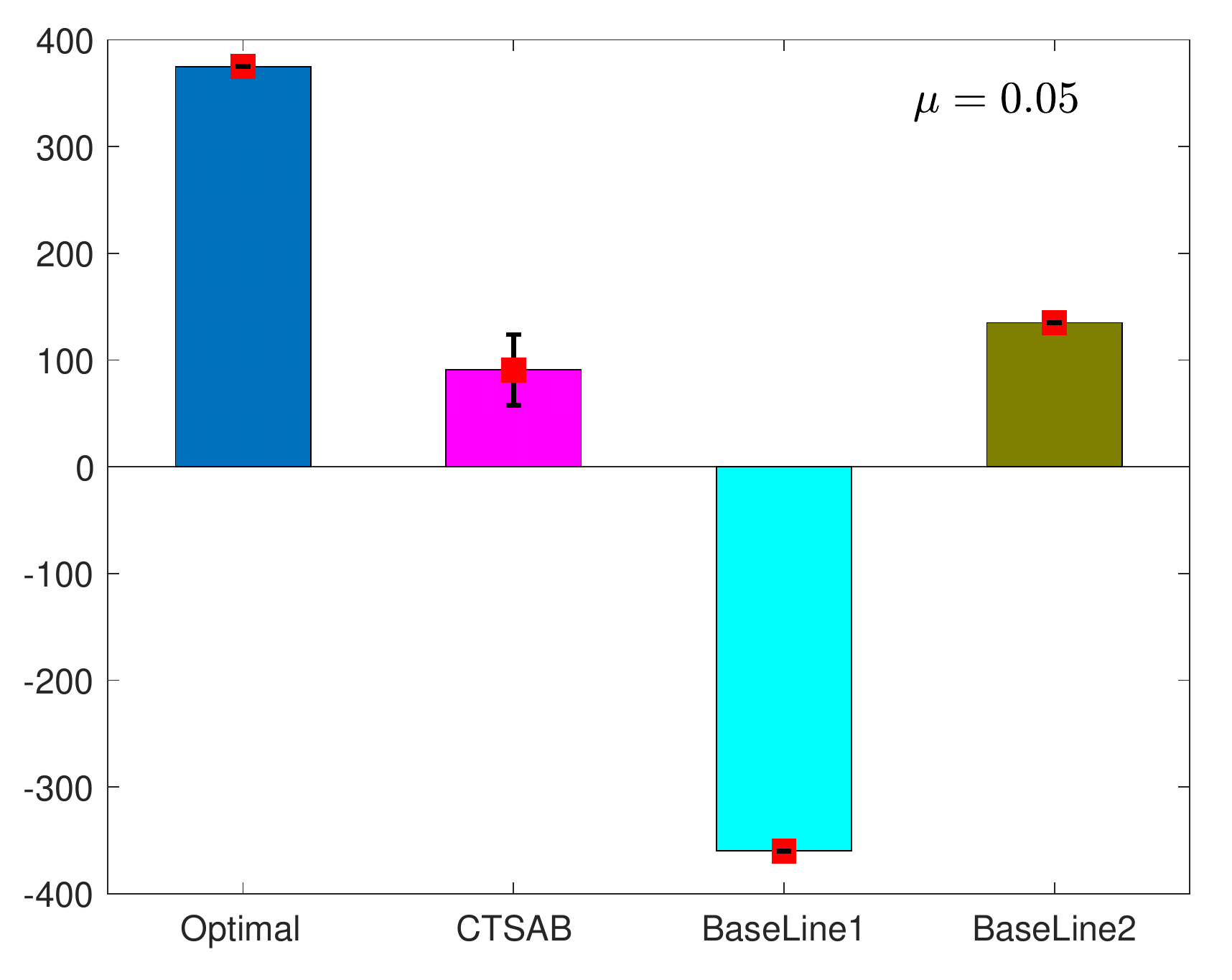}}
	\caption{Comparison of cumulative reward of algorithm {\textsc CTSAB} with other policies for different mean value of the arm. We set $\epsilon=0.05$ and $\delta=0.05$}
	\label{fig:PayoffComparison}
\end{figure*}

\begin{figure}
\centering
\includegraphics[scale=.35]{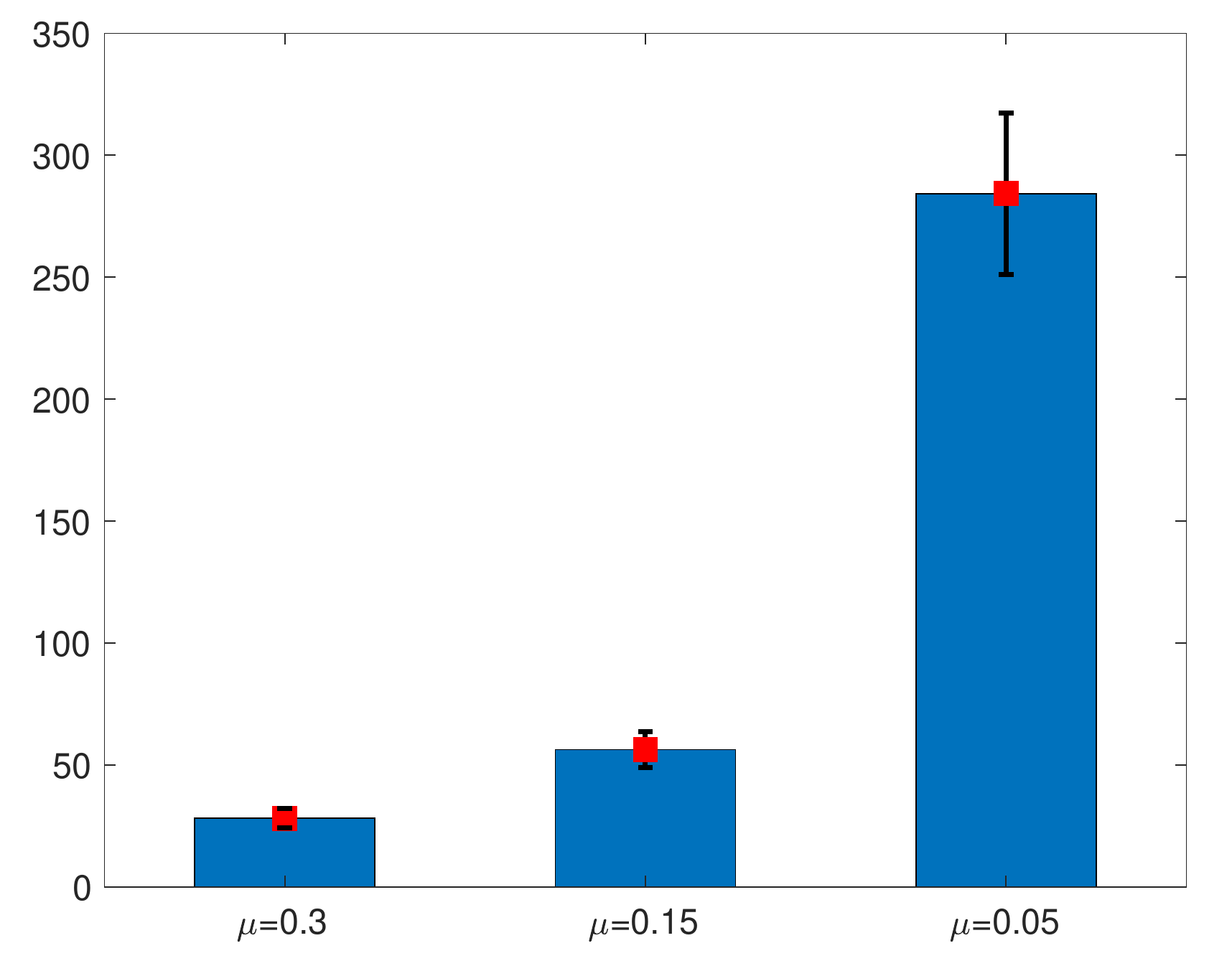}
\caption{Regret of the {\textsc CTSAB} algorithm for the single arm problem as $\mu$ is varied.}
\label{fig:SingleArmDiff}
\end{figure}

\end{document}